\documentclass{article}

\usepackage[main, final]{neurips_2026}

\usepackage[utf8]{inputenc} 
\usepackage[T1]{fontenc}    
\usepackage{hyperref}       
\usepackage{url}            
\usepackage{booktabs}       
\usepackage{amsfonts}       
\usepackage{nicefrac}       
\usepackage{microtype}      
\usepackage{xcolor}         
\usepackage{amsmath}
\usepackage{amsthm}
\usepackage{graphicx}
\usepackage{mathtools}
\usepackage{subcaption}
\usepackage{algorithm}
\usepackage{algorithmic}
\usepackage{booktabs}
\usepackage{multirow}
\usepackage{makecell}
\usepackage{array}
\usepackage[table]{xcolor}
\usepackage{wrapfig}

\newcommand{\best}[1]{\textbf{#1}}

\newcommand{\ms}[2]{#1 {\scriptsize(#2)}}

\newcommand{\redcell}{\rowcolor{purple!10}}
\newcommand{\yellowcell}{\rowcolor{yellow!15}}
\newcommand{\greencell}{\rowcolor{green!20!teal!15}}
\newcommand{\bluecell}{\rowcolor{blue!15!cyan!8}}

\newtheorem{theorem}{Theorem}[section]
\newtheorem{lemma}[theorem]{Lemma}
\newtheorem{assumption}{Assumption}
\newtheorem{proposition}{Proposition}

\title{OTROPE: Optimal Transport-based Robust Off-policy Evaluation for Large Language Models}

\author{
  \begin{tabular}{c@{\hspace{1.2cm}}c@{\hspace{1.2cm}}c}
    Liner Xiang &
    Wenbo Zhang &
    Hengrui Cai\thanks{Corresponding author.} \\
    \texttt{linerx1@uci.edu} &
    \texttt{wenbz13@uci.edu} &
    \texttt{hengrc1@uci.edu} \\
    \vspace{-0.25cm}\\
    \multicolumn{3}{c}{\normalfont
      Department of Statistics, University of California, Irvine}
  \end{tabular}
}

\begin{document}

\maketitle

\begin{abstract}
Reliable evaluation of large language models (LLMs) is essential for their development and deployment, yet is often costly, risky, and difficult to perform safely online. We study \textit{off-policy evaluation} for LLMs, where limited human-labeled data from a behavior model are used to evaluate a newer target LLM. This setting is challenging because labels are scarce, behavior--target distribution shift is common, and response likelihoods are often unavailable for black-box LLMs. We propose the \textbf{O}ptimal \textbf{T}ransport-based \textbf{R}obust \textbf{O}ff-\textbf{P}olicy \textbf{E}valuation (\textbf{OTROPE}), a likelihood-free evaluation that performs \textit{distributional correction} in a \textit{semantic space} via optimal transport to align labeled behavior-policy samples with unlabeled target-policy samples. OTROPE combines corrected human-labeled residuals with proxy predictors, yielding a \textit{doubly robust}-style evaluation without behavior-policy modeling or density-ratio estimation. We theoretically characterize why baseline evaluators fail under LLM distribution shift, and establish consistency and convergence rates for OTROPE when either the reweighted behavior distribution or the proxy predictor converges. Experiments on synthetic and real LLM evaluation tasks show that OTROPE consistently outperforms baselines while enabling ensembles of weaker LLM evaluators to approach and sometimes surpass stronger evaluators. Code is available at \url{https://github.com/LinerXiang/OTROPE}.
\end{abstract}

\section{Introduction}

As large language models (LLMs) continue to advance rapidly, reliable evaluation of their capabilities and behaviors has become increasingly important for their development and deployment. In many practical settings, evaluating an LLM can be viewed as evaluating the policy induced by the language model, which maps user prompts and contexts to generated responses~\citep{bhargava2024off,huang2025pluralistic}. Although online evaluation provides direct evidence of model performance, it is often costly, risky, and difficult to perform safely. Offline evaluation using logged data is thus essential for efficiently and safely assessing the performance of LLMs. This naturally leads to \textit{off-policy evaluation} (OPE)~\citep{sutton1998reinforcement}, a central problem in reinforcement learning and decision-making that estimates the performance of a target policy from data collected under a different behavior policy. For LLMs, OPE provides a principled way to evaluate language models using historical prompts, responses, and feedback, thereby avoiding the cost and risk of direct online experimentation.

Offline evaluation of LLM policies faces three critical challenges. 
\textbf{First, ground-truth labels are often limited.} Evaluating LLMs requires reliable human feedback, but human evaluation is expensive and time-consuming, yielding only limited labeled data. As a scalable alternative, LLMs are increasingly used as proxy evaluators~\citep{zheng2023judging,yang2023gpteval,fu2024gptscore}; yet such proxies are not always trustworthy~\citep{chen2024humans,ye2024justice,schroeder2024can}, making calibration or correction with \textit{limited human labels} necessary. 
\textbf{Second, behavior and target policies often differ.} As LLMs evolve rapidly, human annotations typically collected for earlier models may lag behind the newer models being evaluated, creating behavior–target \textit{distribution shift} (see Figure~\ref{fig:toy_dist}). 
\textbf{Third, such shift is difficult to quantify in LLM settings.}
Response likelihoods are often unavailable as users typically interact with LLMs via APIs or black-box systems. Estimating behavior policies from additional data can introduce policy misspecification and substantial computational cost. Even when such policies are estimated, sequence-level likelihoods for long LLM outputs can become vanishingly small (see Figure~\ref{fig:toy_ratio}).

\begin{figure}[t]
\centering
\begin{subfigure}[t]{0.3\linewidth}
\centering
\includegraphics[width=\linewidth]{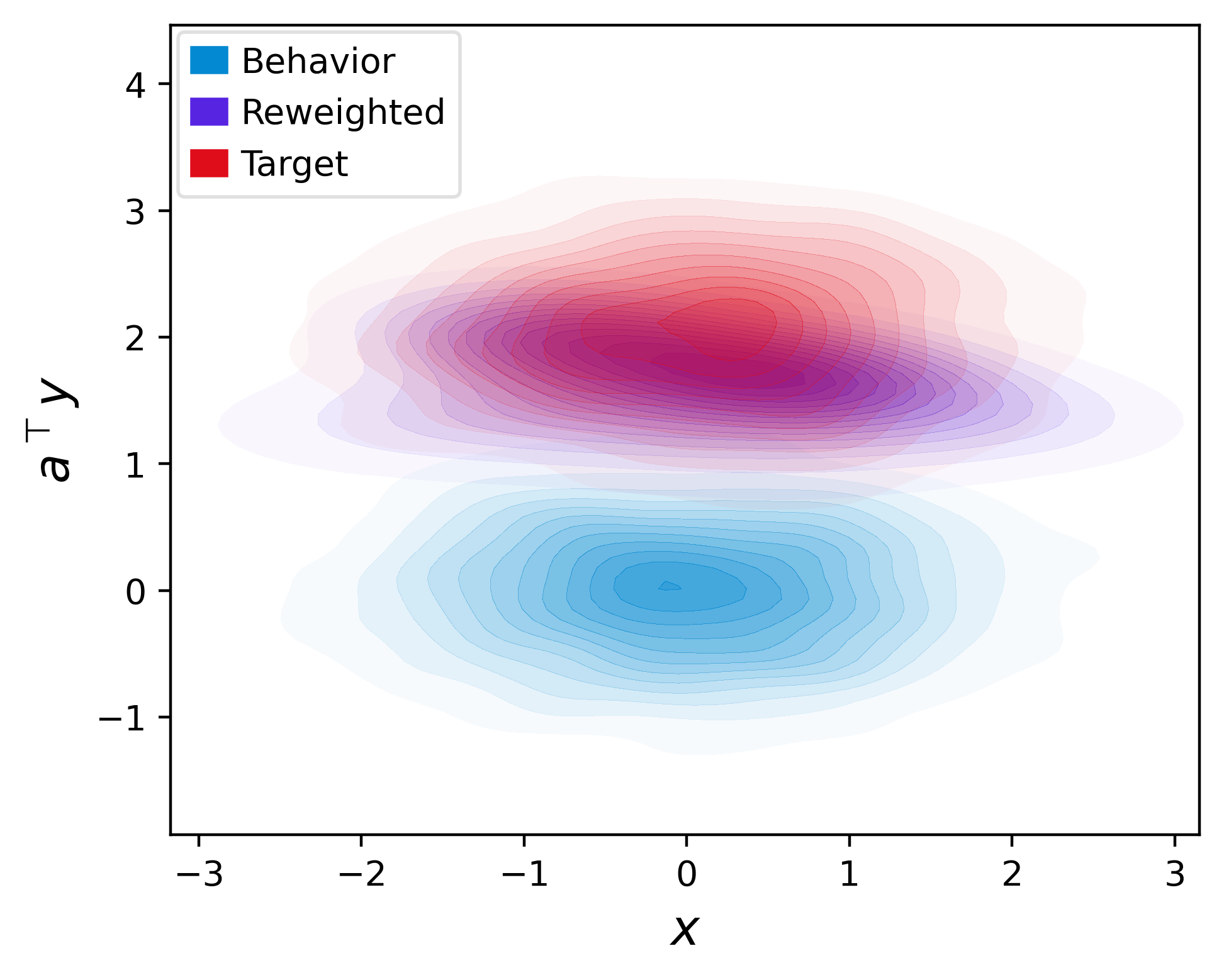}
\caption{}
\label{fig:toy_dist}
\end{subfigure}~
\begin{subfigure}[t]{0.3\linewidth}
\centering
\includegraphics[width=\linewidth]{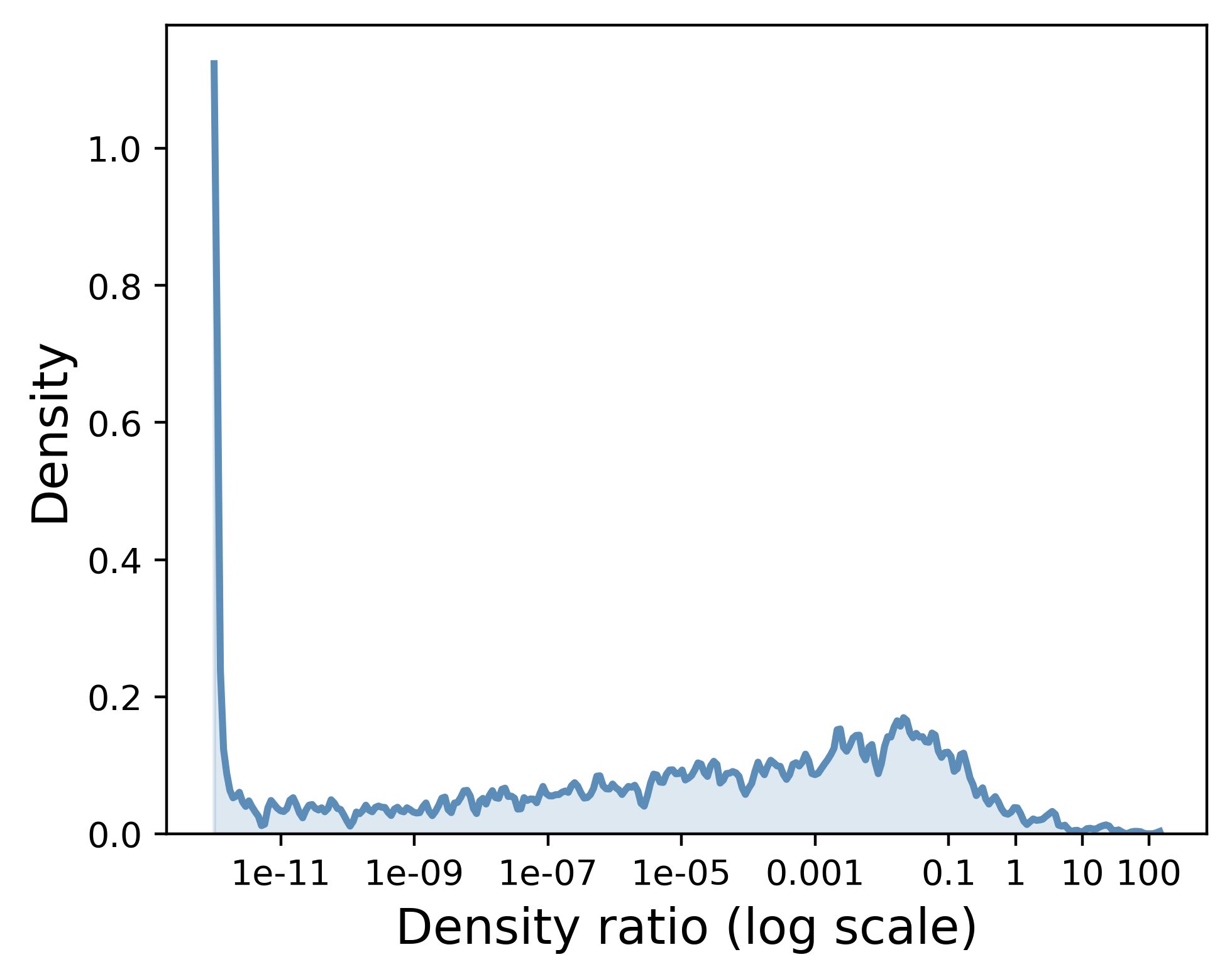}
\caption{}
\label{fig:toy_ratio}
\end{subfigure}~
\begin{subfigure}[t]{0.3\linewidth}
\centering
\includegraphics[width=\linewidth]{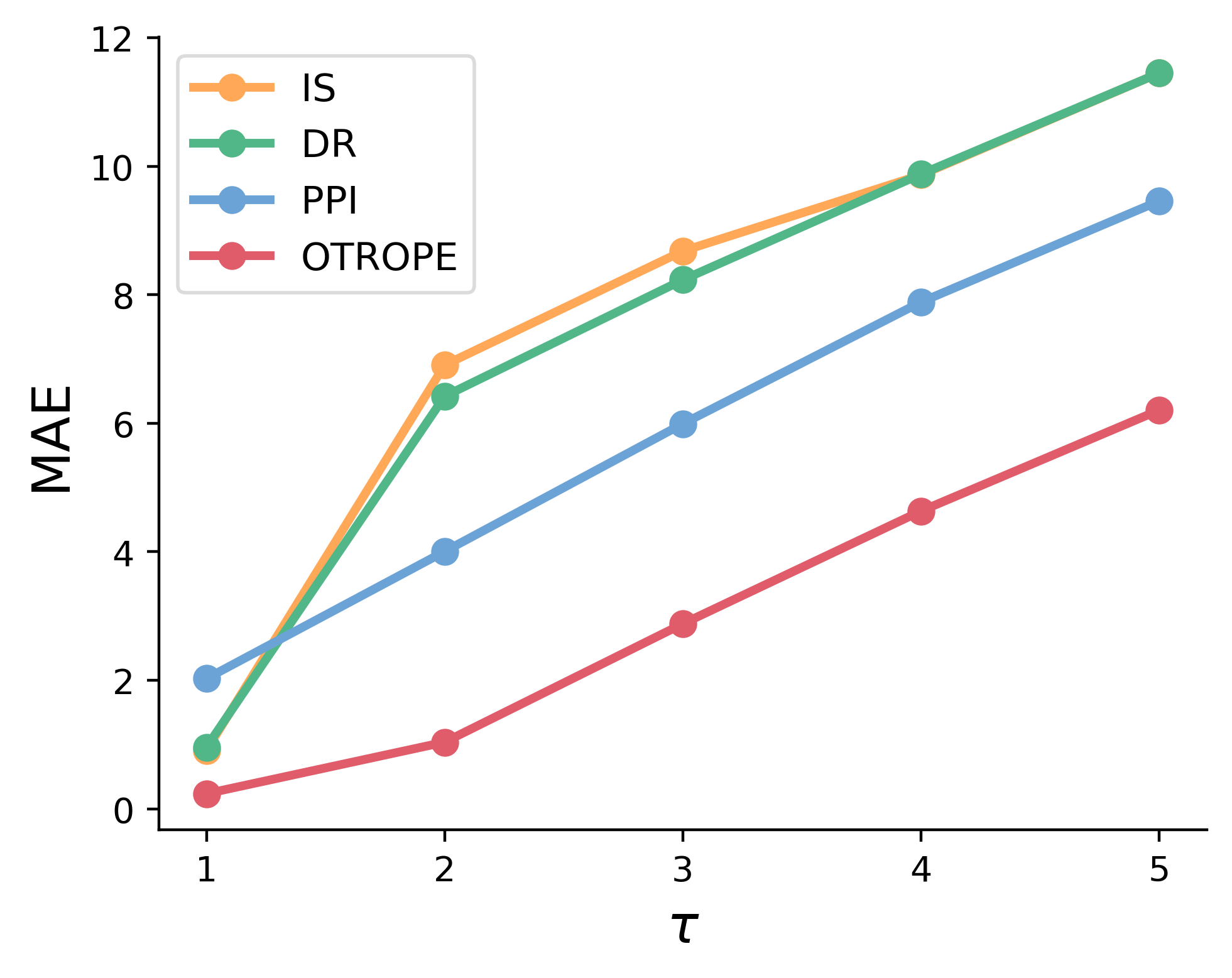}
\caption{}
\label{fig:toy_mae}
\end{subfigure}
\vspace{-0.15cm}
\caption{
(a) Density plots for behavior–target distribution shift, where the reweighted distribution by OTROPE closely matches the target one.
(b) Density curve of behavior–target density ratios on labeled samples, showing a distribution concentrated near zero.
(c) Mean absolute error (MAE) of different estimators with distributional shift level $\tau$. OTROPE consistently achieves the lowest MAE. }
\label{fig:toy}
\vspace{-0.25cm}
\end{figure}

Existing works address only parts of these challenges. Prediction-powered inference (PPI)~\citep{angelopoulos2023prediction,angelopoulos2023ppi++} provides a semi-supervised framework that augments limited ground-truth labels with proxy predictions for unlabeled data. Yet, PPI-type methods typically assume labeled and unlabeled responses follow the same distribution and thus do not handle \textit{distribution shift}. 
Another line of work, importance-weighting-based methods, including importance sampling (IS)~\citep{precup2000eligibility,liu2018breaking} and doubly robust (DR) estimators~\citep{dudik2011doubly,jiang2016doubly,thomas2016data}, can in principle correct for such shift using likelihood or importance ratios. In LLM settings, however, long responses make sequence-level ratios difficult to estimate and highly unstable. Under weak behavior--target overlap and limited samples, the correction terms can degenerate, leading to \textit{substantial bias and high variance} (see Figure~\ref{fig:toy_mae}). 

To address these challenges simultaneously, we propose the \textbf{O}ptimal \textbf{T}ransport-based \textbf{R}obust \textbf{O}ff-\textbf{P}olicy \textbf{E}valuation (\textbf{OTROPE}), a likelihood-free framework for robust off-policy evaluation of LLMs under behavior--target distribution shift. The key insight is to perform distributional correction in a \emph{semantic space} induced by embedding functions, rather than in the original token-level sequence space. In this space, the behavior and target distributions can exhibit substantially greater overlap due to semantic similarity, even when their overlap over response sequences is weak. We then leverage this structure by using optimal transport (OT)~\citep{villani2009optimal} to align labeled behavior-policy samples with unlabeled target-policy samples. The resulting OT weights reweight labeled residuals according to their semantic relevance to the target, and are combined with proxy predictors to produce a robust semi-supervised evaluator for LLMs. 

Our \textbf{contributions} are four-fold: 

$\bullet$ Conceptually, we formulate model-shifted evaluation on the performance of LLM of interest as a semi-supervised OPE problem that jointly addresses limited human annotation and behavior--target distribution shift in the complex LLM settings.\\
$\bullet$ Methodologically, we propose OTROPE that corrects behavior--target distribution shift in a semantic embedding space using OT-based reweighting. Unlike likelihood-ratio-based OPE methods, OTROPE does not require token-level likelihoods, behavior-policy modeling, or sequence-level density-ratio estimation, while enjoys a \textit{doubly robust}-style correction for proxy-based evaluation.\\
$\bullet$ Theoretically, we derive the bias of the PPI estimator under distribution shift, and show the finite-sample degeneracy, bias, and variance inflation of the DR estimator in OPE for LLMs. 
We then establish the \textit{consistency} and \textit{convergence rate} of OTROPE when either the reweighted behavior distribution or the proxy predictor converges. \\
$\bullet$ Empirically, we evaluate OTROPE on synthetic data and real LLM evaluation tasks. OTROPE consistently outperforms baseline methods. In particular, it enables \textit{ensembles of weaker LLM evaluators} to approach and sometimes \textit{outperform stronger LLM judges}, demonstrating the practical value of optimal transport-based correction for scalable LLM evaluation.

\textbf{Related Work.}
We briefly discuss recent work on OPE for LLMs and highlight how our setting differs and our contributions. \citet{bhargava2024off} formulated OPE from logged human feedback, with an emphasis on evaluating preference rankings rather than the quality of response-level generation. More recent studies \citep{wu2024ocean,huang2025pluralistic,xu2026doubly} use OPE-inspired estimators mainly for constructing training objectives for alignment rather than for model evaluation. In contrast, we study \textit{response-level LLM evaluation under model shift with high-dimensional text responses.} 
 Inspired by OT~\citep{villani2009optimal,courty2016optimal,redko2017theoretical,reygner2022reweighting}, we construct a reweighting scheme 
 by minimizing the Wasserstein distance between the behavior distribution and the target distribution. This allows us to correct human-labeled residuals, yielding \textit{a stable correction without requiring explicit density or likelihood ratios}. Unlike prior uses of OT for distributionally robust optimization over uncertainty sets~\citep{shen2023wasserstein} or preference alignment in LLMs~\citep{melnyk2024distributional}, our method is the first work on likelihood-free OPE using OT in LLM settings. 
 We provide a more comprehensive review on OPE, PPI, and OT in Appendix~\ref{app:related_work}.
 
\section{Preliminaries}
\label{sec:pre}
We first formulate the OPE problem in the context of evaluating LLMs, and then review related baselines, including the PPI and DR estimators, discussing their limitations in LLM settings.

\textbf{Problem Setup.} We denote the prompt space and the response space by $\mathcal{X}$ and $\mathcal{Y}$, respectively, and define the joint space as $\mathcal{Z} = \mathcal{X} \times \mathcal{Y}$. Suppose there exists a golden rater (e.g., a human annotator) that provides ground-truth evaluations for prompt--response pairs $z = (x, y) \in \mathcal{Z}$. Formally, let $g^\ast: \mathcal{Z} \to \mathcal{G}$ denote the ground-truth evaluation function, where $\mathcal{G} \subset \mathbb{R}$ is the evaluation space. 

We consider a labeled dataset $\mathcal{D}_1 = \{(x_i, y_i, g_i^\ast)\}_{i=1}^n$ of size $n$, where $x_i \sim \Phi$ denotes a prompt drawn from the distribution $\Phi$, $y_i \sim \pi_0(\cdot \vert x_i)$ is a response generated by a \textit{behavior policy}, and $g_i^\ast = g^\ast(x_i, y_i)$ represents the corresponding ground-truth evaluation. Let $P$ denote the joint distribution of $z = (x,y)$ induced by $x \sim \Phi$ and $y \sim \pi_0(\cdot \vert x)$.
Due to the high cost and time requirements of obtaining reliable human annotations, the labeled sample size $n$ is typically \textit{limited}. In contrast, we have access to a larger unlabeled dataset $\mathcal{D}_2 = \{(\widetilde{x}_j, \widetilde{y}_j)\}_{j=1}^{N}$ of size $N$, where $\widetilde{x}_j \sim \Phi$ and $\widetilde{y}_j \sim \pi(\cdot \vert \widetilde{x}_j)$, with $\pi$ denoting the \textit{target language model or policy}. Let $Q$ denote the corresponding joint distribution of $\widetilde{z} = (\widetilde{x},\widetilde{y})$. Ground-truth evaluations are \textit{unavailable} for samples in $\mathcal{D}_2$.
Throughout this work, we consider autoregressive language models, in which the likelihood of a response is factorized over token-level probabilities. 
Since $\mathcal{D}_2$ lacks ground-truth human evaluations, we adopt the \emph{LLM-as-a-judge} paradigm to construct \textit{proxy evaluators} that evaluate each prompt--response pair in $\mathcal{D}_2$, denoted as $\widehat{g}(\widetilde{z}_j)$ (with more details provided in Section~\ref{subsec:our}). 
We further define the residual function $r(\widetilde{z}) = g^\ast(\widetilde{z}) - \widehat{g}(\widetilde{z})$.
Our estimand is \textit{the target policy value}:
\begin{equation}
    \label{eq:value_def}
    V^\ast(\pi) = \mathbb{E}_{\widetilde{z} \sim Q} \, [g^\ast(\widetilde{z})] = \mathbb{E}_{\widetilde{x} \sim \Phi} \, \mathbb{E}_{\widetilde{y} \sim \pi(\cdot \vert \widetilde{x})} \, [g^\ast(\widetilde{x},\widetilde{y})].
\end{equation}
Our goal is to estimate $V^\ast(\pi)$ based on the datasets $\mathcal{D}_1$ and $\mathcal{D}_2$, together with the predictor $\widehat{g}$.

\textbf{Prediction-Powered Inference (PPI) Estimator.} After obtaining the proxy evaluations, a natural approach is to estimate \(V^\ast(\pi)\) directly using these predictions. This leads to the direct method (DM)~\citep{li2010contextual}, defined as $\widehat{V}_{\mathrm{DM}}(\pi)
    = 1/N\sum_{j=1}^{N} \widehat{g}(\widetilde{z}_j)$.
However, such proxy evaluations may be unreliable. 
To mitigate prediction error, PPI~\citep{angelopoulos2023prediction} combines a small amount of labeled data with a larger set of unlabeled data with arbitrary proxy predictors. Specifically, we consider the following PPI-based estimator for the policy value:
\begin{equation}
    \label{eq:ppi_value}
    \widehat{V}_{\mathrm{PPI}}(\pi)
    = \frac{1}{N}\sum_{j=1}^{N} \widehat{g}(\widetilde{z}_j)
    + \frac{1}{n} \sum_{i=1}^n 
    \left( g^\ast(z_i) - \widehat{g}(z_i) \right),
\end{equation}
where $z_i = (x_i, y_i) \in \mathcal{D}_1$ and $\widetilde{z}_j = (\widetilde{x}_j, \widetilde{y}_j) \in \mathcal{D}_2$. The second term serves as a correction term that accounts for the prediction error of $\widehat{g}$ using labeled data.
Nevertheless, PPI assumes that \(\mathcal{D}_1\) and \(\mathcal{D}_2\) are drawn from the same distribution. Under \textit{distribution shift}, i.e., $\pi_0 \neq \pi$, $\widehat{V}_{\mathrm{PPI}}(\pi)$ is unbiased only if the prediction error is invariant across distributions; otherwise, it becomes biased and unreliable.

\textbf{Doubly Robust (DR) Estimator.} To address distribution shift in OPE, the DR method~\citep{jiang2016doubly} is widely adopted that combines outcome modeling with behavior–target density ratio modeling \citep{precup2000eligibility} (i.e., $\pi /\pi_0$).  Following the PPI framework, a DR-based value estimator is
\begin{equation}
    \label{eq:dr_value}
    \widehat{V}_{\text{DR}}(\pi)=\frac{1}{N}\sum_{j=1}^{N}  
   \widehat{g}(\widetilde{x}_j, \widetilde{y}_j) + \frac{1}{n} \sum_{i=1}^n 
   \frac{\pi(y_i \vert x_i)}{\pi_0(y_i \vert x_i)}
   \left( g^\ast(x_i, y_i) - \widehat{g}(x_i, y_i) \right),
\end{equation}
where the inner part is the IS estimator $\widehat{V}_{\text{IS}}(\pi)= (1/n) \sum_{i=1}^n 
\{ {\pi(y_i \vert x_i)}/{\pi_0(y_i \vert x_i)} \} g^\ast(x_i, y_i).$ 
In the context of LLMs, 
sequence probabilities are products of token-level probabilities, so small differences can accumulate into weak overlap between behavior and target policies, leading to degenerate density ratios and unreliable importance weighting. Moreover, in black-box API settings, these likelihoods are often unavailable, further complicating policy evaluation.

\section{Optimal Transport-based Robust Off-Policy Evaluation}
In this section, we first use a toy example to illustrate the limitations of the PPI and DR methods in Section~\ref{subsec:toy}, where PPI can be biased under distribution shift, whereas DR can degenerate under weak overlap in high-dimensional LLM sequence spaces. This motivates our robust likelihood-free approach to response-level shift-aware residual correction, as presented in Section~\ref{subsec:our}.

\subsection{Toy Example}
\label{subsec:toy}

We generate the prompt covariate $x$ from a truncated normal distribution. The behavior and target distributions, $P$ and $Q$, are also parameterized as truncated normal distributions, with parameters given as functions of $x$. We utilize parameter $\tau \ge 0$ to control the strength of the distributional shift between $P$ and $Q$. We use a misspecified predictor $\widehat{g}$, which induces heterogeneous residuals under $P$ and $Q$. We set $n = 1000$ and $N = 2000$, and repeat the experiment over $200$ independent replications. The detailed setup is provided in Section~\ref{subsec:simu}. As shown in Figure~\ref{fig:toy_dist}, the overlap between the behavior and target distributions becomes very weak when the shift parameter $\tau=2$. As a result, the density ratio is close to zero with high probability (Figure~\ref{fig:toy_ratio}), making the DR correction term negligible under finite sample sizes. Figure~\ref{fig:toy_mae} reports the mean absolute error (MAE) of different estimators. Density ratio-based estimators, including IS and DR estimators, perform poorly, highlighting the ineffectiveness of importance weighting in this weak-overlap regime. Moreover, residual heterogeneity across $P$ and $Q$ leads to poor performance of PPI-based methods. Detailed results of MAE under varying distribution shift are reported in Appendix~\ref{app:toy}.

Motivated by these observations, we propose a new weighting scheme that combines the residual-correction structure of PPI and DR with distributional reweighting in the embedding space. 
Figure~\ref{fig:toy_dist} shows that our proposed method OTROPE reweights the behavior distribution to closely match the target distribution; moreover, Figure~\ref{fig:toy_mae} demonstrates that OTROPE consistently achieves the lowest MAE as the distribution shift increases with $\tau$, highlighting its robustness to shift.

\subsection{Proposed Method: OTROPE}
\label{subsec:our}

To address the negligible correction induced by near-zero density ratios, we consider an alternative approach based on OT to reweight the labeled residuals. Intuitively, OT seeks the most cost-efficient way to transform one probability measure into another. To define the transport cost between two prompt-response pairs, we shift from exact token-level distribution to a \emph{semantic geometric distribution} framework by considering a semantic embedding function $e:\mathcal{Z} \rightarrow \mathbb{R}^d$, and denote its image by $e(\mathcal Z)=\{e(z):z\in\mathcal Z\}\subseteq \mathbb R^d$. We denote by \(e_\# P\) and \(e_\# Q\) the pushforward measures of \(P\) and \(Q\) under \(e\), respectively. Let $c(u,v)=\|u-v\|$, be the Euclidean cost on \(\mathbb R^d\). The OT cost between the pushforward measures \(e_\# P\) and \(e_\# Q\) is defined as
\begin{equation*}
    \mathcal{T}_{c}(e_\# P,e_\# Q) \coloneq \inf_{\gamma \in \Gamma(e_\# P,e_\# Q)} \int_{e(\mathcal Z) \times e(\mathcal Z)} c(u,v) \, \mathrm{d}\gamma(u,v),
\end{equation*}
where $\Gamma(e_\# P,e_\# Q)$ denotes the set of all couplings between $e_\# P$ and $e_\# Q$, i.e., the set of probability measures on $e(\mathcal Z) \times e(\mathcal Z)$ with marginals $e_\# P$ and $e_\# Q$.
The corresponding Wasserstein distance~\citep{vaserstein1969markov}, also known as the Earth Mover's Distance~\citep{rubner2000earth}, is defined as $W(\mu, \nu) \coloneq \mathcal{T}_{c}(\mu,\nu)$.
The Wasserstein distance admits a natural interpretation as the minimal transportation cost required to move probability mass from $e_\# P$ to $e_\# Q$ under the cost function $c$. In LLM evaluation settings, the pushforward measures \(e_\# P\) and \(e_\# Q\) are unknown and must be approximated from samples. Given empirical samples in $\mathcal{D}_1$ and $\mathcal{D}_2$, we define
$e_\# P_n=1/n\sum_{i=1}^n \delta_{e(z_i)}$, and $e_\# Q_N=1/N\sum_{j=1}^N \delta_{e(\widetilde z_j)}$, where \(\delta_{e(z)}\) denotes the Dirac measure at \(e(z)\).
Instead of relying on importance weight in the response space to address distribution shift, we propose to reweight the labeled samples in $\mathcal{D}_1$ within a semantic space such that the resulting weighted distribution aligns with the empirical distribution of the target samples in $\mathcal{D}_2$ under the Wasserstein distance. The resulting weighted residuals over $e_\# P_n$ can then serve as an approximation to the mean residual under $e_\# Q_N$. 
We formulate the weight estimation problem as an optimal transport minimization problem over reweighted empirical measures. Let $\Delta_n = \left\{ \boldsymbol{w} \in \mathbb{R}^n \;\middle|\; \sum_{i=1}^n w_i = 1,\; w_i \ge 0 \right\}$ denote the probability simplex. 
We aim to determine optimal weights $\boldsymbol{w}^\ast = (w_1^\ast, \dots, w_n^\ast)$ such that the reweighted labeled distribution closely matches the unlabeled distribution in terms of the Wasserstein distance. Specifically, we solve
\begin{equation}
\label{eq:OT_weights}
    \boldsymbol{w}^\ast \in \underset{\boldsymbol{w} \in \Delta_n}{\arg\min} \;
    W \left( \sum_{i=1}^n w_i \, \delta_{e(z_i)}, \; \frac{1}{N} \sum_{j=1}^{N} \delta_{e(\widetilde{z}_j)} \right)=\underset{\boldsymbol{w} \in \Delta_n}{\arg\min} \;\mathcal{T}_{c}(e_\#P_{n,\boldsymbol{w}},e_\#Q_N),
\end{equation}
where $e_\#P_{n,\boldsymbol w} \coloneq \sum_{i=1}^n w_i \delta_{e(z_i)}$ is the reweighted labeled distribution.
This formulation seeks weights that minimize the transportation cost required to align the labeled behavior distribution with the target distribution under the embedding-induced geometry. In the empirical setting, this problem reduces to a finite-dimensional optimal transport problem over samples, given by
\begin{equation}
\label{equ:solve_weights}
\min_{\boldsymbol{w} \in \Delta_n,\;\Gamma \in \mathbb{R}_+^{n\times N}}
\sum_{i=1}^n \sum_{j=1}^N 
\Gamma_{ij}\, c\big(e(z_i), e(\widetilde z_j)\big),
\quad
\text{s.t. }
\sum_{j=1}^N \Gamma_{ij}=w_i,\;
\sum_{i=1}^n \Gamma_{ij}=\frac{1}{N},
\end{equation}
where $\Gamma \in \mathbb{R}_+^{n \times N}$ denotes the transport plan and $\Gamma_{ij}$ represents the mass transported from $z_i$ to $\widetilde z_j$. This is a linear program over $(\boldsymbol w, \Gamma)$, and can be efficiently solved using standard OT solvers or via entropic regularization~\citep{cuturi2013sinkhorn}.
Replacing density ratios in Eq.~\eqref{eq:dr_value} with these geometry-aware weights, we propose the \textbf{O}ptimal \textbf{T}ransport-based \textbf{R}obust \textbf{O}ff-\textbf{P}olicy \textbf{E}valuation (\textbf{OTROPE}):
\begin{equation}
\label{eq:ot_value}
   \widehat{V}_{\mathrm{OTROPE}}(\pi) = \frac{1}{N}\sum_{j=1}^{N} \widehat{g}(\widetilde{z}_j)
   + \sum_{i=1}^n w_i^\ast \left( g^\ast(z_i) - \widehat{g}(z_i) \right).
\end{equation}
Intuitively, as characterized by Eq.~\eqref{eq:OT_weights}, a labeled sample \(z_i\) receives a larger weight \(w_i^\ast\) when it is closer in the semantic embedding space to regions where the target policy \(\pi\) places higher probability mass. As a result, its residual \(g^\ast(z_i)-\widehat g(z_i)\) contributes more strongly to correcting the proxy prediction under the target distribution. 
We summarize OTROPE in Algorithm~\ref{alg:ot} and detail OT weights training in Appendix~\ref{app:ot_optimize}. In practice, we can employ $M$ small LLMs as evaluators, denoted by $\{g_i\}_{i=1}^M$,  and aggregate their evaluations for each pair in $\mathcal{D}_2$ to construct $\widehat{g}$. 
We consider two aggregation strategies. 
The first is \textit{mixture-of-experts} (MoE), which forms a weighted combination of judgments (detailed in Appendix~\ref{sec:moe}).
The second is \textit{majority voting} for pair preference, which uses the majority label among evaluators as the proxy prediction. We show in Section~\ref{sec:exper} that using OTROPE, \textit{ensembles of weaker LLMs} can approach and sometimes \textit{outperform stronger LLM judges}.

\begin{algorithm}[t]
\caption{OTROPE: Optimal Transport-based Robust Off-Policy Evaluation}
\label{alg:ot}
\begin{algorithmic}[1]
\STATE \textbf{Input:} labeled dataset 
\(\mathcal{D}_1=\{(z_i,g_i^\ast)\}_{i=1}^n\) from the unknown behavior policy $\pi_0$, unlabeled dataset 
\(\mathcal{D}_2=\{\widetilde z_j\}_{j=1}^N\) from the target policy $\pi$, embedding map \(e(\cdot)\), transport cost \(c(\cdot,\cdot)\), and a predictor \(\widehat g\), which is either pre-specified or trained on an \(\mathcal{D}_1\).
\STATE Obtain predictions \(\widehat g(z_i)\) for \(z_i\in\mathcal{D}_1\) and \(\widehat g(\widetilde z_j)\) for \(\widetilde z_j\in\mathcal{D}_2\).
\STATE Construct embeddings for both samples: $\{e(z_i)\}_{i=1}^n$ and $\{e(\widetilde z_j)\}_{j=1}^N$.
\STATE Form the pairwise cost $c(e(z_i),e(\widetilde z_j))$, for $i=1,\ldots,n$, $j=1,\ldots,N$.
\STATE Optimize the OT weighting problem in~\eqref{equ:solve_weights} to obtain the weights $\boldsymbol{w}^\ast$.
\STATE Compute the \(\widehat{V}_{\mathrm{OTROPE}}(\pi)\) by~\eqref{eq:ot_value}.
\STATE \textbf{return} \(\widehat{V}_{\mathrm{OTROPE}}(\pi)\).
\end{algorithmic}
\end{algorithm}

\section{Theoretical Results}
\label{sec:theory}
In this section, we first theoretically analyze the limitations of PPI and DR estimators, and then establish the consistency and convergence rate of the proposed OTROPE. 

\subsection{Limitations of PPI and DR Estimators}
The following proposition shows that the PPI estimator in Eq.~\eqref{eq:ppi_value} is generally biased when there exists discrepancy between the behavior distribution \(P\)  and the target distribution \(Q\).

\begin{proposition}[Bias of the PPI Estimator under Distribution Shift] 
\label{prop:ppi_bias}
If $\mathbb{E}_{z\sim P}[r(z)] \neq \mathbb{E}_{z\sim Q}[r(z)]$, then
\[
\left| \mathbb{E} \big[\widehat{V}_{\mathrm{PPI}}(\pi)\big] - V^\ast(\pi) \right| =
\left| \mathbb{E}_{z\sim P} [r(z)] - \mathbb{E}_{z\sim Q} [r(z)] \right| > 0.
\]
\end{proposition}
This follows immediately by taking expectation: $\mathbb{E} \big[\widehat{V}_{\mathrm{PPI}}(\pi)\big] =
\mathbb{E}_{z\sim Q} [\widehat{g}(z)] + \mathbb{E}_{z\sim P} \big[g^\ast(z) - \widehat{g}(z)\big]
= \mathbb{E}_{z\sim Q} [\widehat{g}(z)] + \mathbb{E}_{z\sim P} \big[r(z)\big]$.
This result reflects an intrinsic limitation due to the discrepancy between $P$ and $Q$.
We next examine the limitations of density-ratio based DR estimator in Eq.~\eqref{eq:dr_value}.

\begin{proposition}[Finite-sample degeneracy, bias, and variance inflation of DR]
\label{prop:dr_bad}
Denote the density ratio by $\rho(z_i)={\pi(y_i \vert x_i)}/{\pi_0(y_i \vert x_i)}$. Suppose \(0<c_r\le r(z)\le C_r\) for all \(z\sim P\).
If there exist sequences \(a_n\in(0,1)\), \(b_n>0\), and \(d_n\in(0,1)\), such that $\mathbb P\left(1/n\sum_{i=1}^n \rho(z_i)\le a_n\right)\ge 1-d_n$,
and $\mathbb P(|\widehat V_{\rm DM}(\pi)-V^\ast(\pi)|\ge b_n)\ge 1-d_n$,
then, 
\(
\left| \widehat V_{\rm DR}(\pi)-V^\ast(\pi) \right| \ge b_n-C_r a_n,
\) with probability at least \(1-2d_n\). 
If additionally $|\operatorname{Cov}(\widehat V_{\rm DM},R_n)|
\le \kappa \operatorname{Var}(R_n)$ with $\kappa<1/2$, then
\[
\operatorname{Var}\{\widehat V_{\rm DR}(\pi)\} \ge (1-2\kappa)
\left[c_r^2\frac{(1-a_n)^2}{d_n} - C_r^2 \right].
\]
\end{proposition}

Proposition~\ref{prop:dr_bad} formalizes the finite-sample failure of DR under weak overlap. When empirical density ratios are close to zero with high probability, the residual correction becomes negligible, leading to biased estimates. This degeneracy implies occasional large weights, resulting in high variance, which is frequently encountered in LLMs applications. Proofs are provided in Appendix~\ref{app:proof_prop_bad}.

\subsection{Consistency and Convergence Rate of OTROPE}
In contrast, we show that the proposed OTROPE estimator consistently recovers the true value $V^\ast(\pi)$ under suitable conditions. 
We begin by introducing Assumption~\ref{ass:bound_emb}.

\begin{assumption}[Overlap]
\label{ass:bound_emb}
$\operatorname{supp}(e_\#Q)\subseteq \operatorname{supp}(e_\#P)$. $\operatorname{supp}(e_\#P)$ and $\operatorname{supp}(e_\#Q)$ are compact.
\end{assumption}
Assumption~\ref{ass:bound_emb} is standard in the OT literature~\citep{gangbo1996geometry,fournier2015rate,manole2024sharp}. In our setting, it can be interpreted as requiring overlap after dimensionality reduction via the embedding map, rather than in the original high-dimensional sequence space as in DR methods. 
This assumption is reasonable in LLM settings, where both $P$ and $Q$ arise from policies defined over the same prompt distribution and share a common generation space. 
Moreover, embedding representations of LLM-generated responses tend to concentrate within a bounded region of $\mathbb{R}^d$, further supporting the validity of this assumption.

\begin{lemma}
\label{lemma}
Define
\(
\mathcal{M}_n \coloneq
\left\{
\sum_{i=1}^n w_i \delta_{e(z_i)} : w_i \ge 0,\ \sum_{i=1}^n w_i = 1
\right\}.
\) Suppose Assumption~\ref{ass:bound_emb} holds. 
Then, as $n \to \infty$,
\[
\inf_{\mu \in \mathcal{M}_n} W(\mu, e_\#Q) \xrightarrow{\text{a.s.}} 0.
\]
\end{lemma}
Lemma~\ref{lemma} shows that, as the sample size of $\mathcal{D}_1$ increases, reweighted distributions can approximate the target distribution $Q$ well in the Wasserstein distance. The proof is provided in Appendix~\ref{app:proof_lemma}.

\begin{assumption}[Lipschitz Continuity]
\label{ass:lr}
There exists a function $\widetilde r:\mathbb R^d\to \mathbb R$ such that $r(z)=\widetilde r(e(z))$ for all $z\in\mathcal Z$. Moreover, there exists a constant $L_r>0$ such that $\widetilde r$ is $L_r$-Lipschitz, i.e., $|\widetilde r(u)-\widetilde r(v)| \le L_r\|u-v\|$ for all $u,v\in e(\mathcal Z)$.
\end{assumption}

Assumption~\ref{ass:lr} is a standard regularity condition~\citep{manole2024sharp,hundrieser2024empirical,melnyk2024distributional}, requiring the residual function to vary smoothly in the semantic space. The Lipschitz continuity ensures that samples that are close in the embedding space, and hence semantically similar, exhibit similar prediction errors, thereby enabling effective bias correction. 

\begin{theorem}[Consistency of OTROPE]
\label{the:unbias}
Suppose Assumptions~\ref{ass:bound_emb} and~\ref{ass:lr} hold. Then, as $n \to \infty$ and $N \to \infty$,
\[
\widehat{V}_{\mathrm{OTROPE}}(\pi) \xrightarrow{p}  V^\ast(\pi).
\]
\end{theorem}

Since the residual function $\widetilde{r}(e(z))$ is $L_r$-Lipschitz continuous, Lemma~\ref{lemma} implies that expectations of $\widetilde{r}$ under $e_\#P_{n,\boldsymbol w^\ast}$ and $e_\#Q$ become arbitrarily close. Consequently, the bias introduced by evaluating the correction term under $P_{n,\boldsymbol w^\ast}$ vanishes asymptotically. Detailed proofs are in Appendix~\ref{app:proof_unbias}. 
Unlike the PPI estimator in Eq.~\eqref{eq:ppi_value}, which uses the average residual under $P$ and remains biased even under Assumptions~\ref{ass:bound_emb} and~\ref{ass:lr}, OTROPE in Eq.~\eqref{eq:ot_value} reweights $e_\# P_n$ to better align with $e_\# Q_N$. 
Under a lower mass condition, we further establish the convergence rate of OTROPE.

\begin{assumption}[Lower Mass Condition]
\label{ass:ball}
There exist constants $c>0$, $0 < d_0 \le d$, and $\eta_0>0$ such that for all
$u\in \operatorname{supp}(e_\#Q)$ and all $0<\eta\le \eta_0$, $e_\#P(B(u,\eta)) \ge c \eta^{d_0}$,
where $e_\#P(B(u,\eta))$ denotes the probability mass assigned by the pushforward distribution $e_\#P$ to the Euclidean ball centered at $u$ with radius $\eta$ in the embedding space.
\end{assumption}

Assumption~\ref{ass:ball} is a standard lower mass condition that guarantees sufficient local mass around the support of the target distribution, and is commonly used to derive convergence rates in OT analyses~\citep{weed2019sharp,reygner2022reweighting,hundrieser2024empirical}. 

\begin{theorem}[Convergence Rate of OTROPE]
\label{the:ot_rate}
Suppose Assumptions~\ref{ass:bound_emb},~\ref{ass:lr} and~\ref{ass:ball} hold. If $\text{Var}_{\widetilde{z}\sim Q}(\widehat{g}(\widetilde{z}))$ is bounded, then
\[
\widehat{V}_{\mathrm{OTROPE}}(\pi)-V^\ast(\pi)
= O_p\!\left( N^{-1/2} + n^{-1/d_0} + N^{-1/d} \right).
\]
\end{theorem}
$d_0$ can be interpreted as the intrinsic dimension~\citep{weed2019sharp} of the support of the distribution in the embedding space, reflecting how the probability mass of $e_\# P$ concentrates locally. Moreover, the convergence rate we establish is at most $O_p(N^{-1/d})$, implying that embeddings on lower-dimensional manifolds yield faster convergence. This highlights the importance of selecting embedding representations that capture essential semantic structure while remaining low-dimensional. 

More importantly, OTROPE enjoys a form of \textit{double robustness}, as shown in the following theorem.

\begin{theorem}[Double Robustness of OTROPE]
\label{the:double}
Suppose $\|\widehat g_n - g^\ast\|_\infty = \sup_{z \in \mathcal Z}|\widehat g_n(z)-g^\ast(z)|= O_p(\alpha_n)$ with $\alpha_n \to 0$, and $\widehat g_n$ has bounded conditional variance under the target distribution. Then, for any possibly data-dependent nonnegative weights $\widehat w_i$ satisfying $\widehat w_i \ge 0$ and $\sum_{i=1}^n \widehat w_i = 1$, we have
\[
\widehat V_{\mathrm{OTROPE}}(\pi) - V^\ast(\pi)
=
O_p\left(N^{-1/2} + \alpha_n\right).
\]
\end{theorem}

The double robustness follows by combining Theorem~\ref{the:ot_rate} with Theorem~\ref{the:double}. As discussed earlier, when the optimal transport weights $\boldsymbol{w}^\ast$ solve Eq.~\eqref{eq:OT_weights} and the residual induced by $\widehat g$ satisfies Assumption~\ref{ass:lr}, OTROPE is consistent and achieves the convergence rate stated in Theorem~\ref{the:ot_rate}. Conversely, Theorem~\ref{the:double} shows that if the proxy predictor $\widehat g_n$ is sufficiently accurate, then OTROPE remains consistent even when the transport weights are not accurate, as long as they are nonnegative and normalized. 
This is because the weighted correction term is bounded by $O_p(\alpha_n)$ under the normalization of the weights. In particular, when $\alpha_n = O(n^{-1/d_0})$, OTROPE achieves the rate $O_p\left(N^{-1/2}+n^{-1/d_0}\right)$, which avoids the additional $N^{-1/d}$ term appearing in Theorem~\ref{the:ot_rate}. Proofs of Theorems~\ref{the:ot_rate} and~\ref{the:double} are provided in Appendix~\ref{app:proof_rate}.

\section{Experiments}
\label{sec:exper}
We evaluate OTROPE on both controlled experiments with various distribution shifts and real application on LLM evaluations, with comparison to PPI-based and density-ratio-based baselines. Beyond PPI, IS, and DR introduced in Section~\ref{sec:pre}, we further include  PPI++~\citep{angelopoulos2023ppi++}. Baselines are detailed in Appendix~\ref{app:baseline}, and hyperparameter choices are detailed in Appendix~\ref{asec:hyper}.

\subsection{Controlled Experiments with Various Distribution Shifts}
\label{subsec:simu}

\begin{figure}[!t]
    \centering
    \includegraphics[width=0.9\linewidth]{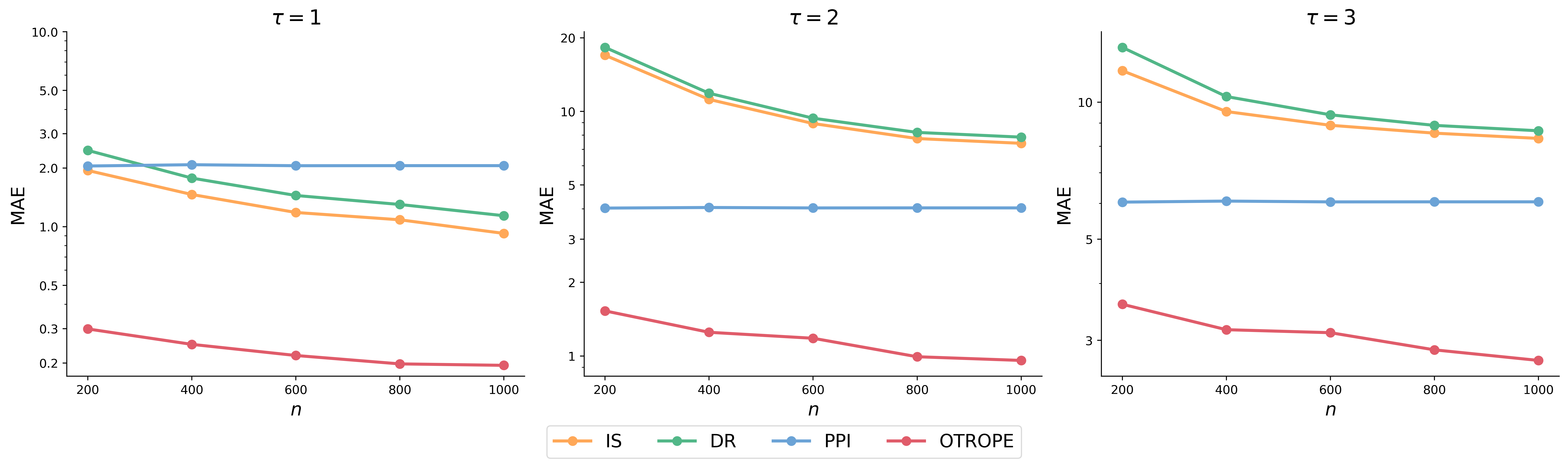}
    \caption{Log-scale mean absolute error under a fixed predictor $\widehat g$ for distribution shift levels $\tau \in \{1,2,3\}$. PPI++ is omitted since it nearly overlaps with PPI.}
    \label{fig:simu}
\end{figure}

\textbf{Settings.} 
We first generate a covariate \(x \sim \mathrm{TruncNormal}(0,1^2;[-3,3])\).
Given \(x\), we generate a \(d\)-dimensional response \(y \in \mathbb{R}^d\), where we set \(d=5\).
Under the behavior distribution \(P\), given \(x_i\), the coordinates of \(y_i\) are conditionally independent and satisfy
\(
y_{ik} \vert x_i \sim \mathrm{TruncNormal}(0.5\, x_i,\,0.5^2;[-4,4]), k=1,\ldots,d.
\)
Under the target distribution \(Q\), we consider a distributional shift along a outcome-relevant direction \(a \in \mathbb{R}^d\), where \(\|a\| = 1\).
Specifically, given \(\widetilde{x}_j\),
\(
\widetilde{y}_{jk} \vert \widetilde{x}_j \sim \mathrm{TruncNormal}(0.5\, \widetilde{x}_j + \tau \times a_k,\,0.5^2;[-4,4]), k=1,\ldots,d.
\)
The parameter \(\tau \ge 0\) controls the strength of the distributional shift between \(P\) and \(Q\).
The ground-truth outcome is given by
\(
g^\ast(x,y) = 2 + 3x + 2 a^\top y + \varepsilon, \varepsilon \sim N(0,1^2).
\)
We consider a misspecified fixed predictor $\widehat{g}(x,y) = -x$ used in Section~\ref{subsec:toy}, which cannot capture the distributional shift induced by $y$, leading to biased estimates under the target distribution. 
Following \citet{angelopoulos2023ppi++}, we fix $N=2000$, vary the labeled sample size $n$, and repeat each experiment over $100$ replications. We focus on distribution shift levels $\tau \in \{1,2,3\}$. We also consider a predictor $\widehat g(x,y)$ trained on $\mathcal D_1$, with details provided in Appendix~\ref{app:simu}. For OTROPE, we use the embedding $e(x,y)=(x,a^\top y)\in\mathbb R^2$.

\textbf{Results.} 
Figure~\ref{fig:simu} reports the MAE for $\tau \in \{1,2,3\}$. In our experiments, PPI++ performs almost identically to PPI and is therefore omitted from the plots for clarity. Since the prediction model $\widehat g$ is fixed, the PPI-based methods exhibit little improvement as $n$ increases. In contrast, OTROPE consistently outperforms all baselines across all settings. Under large distribution shifts, i.e., $\tau=3$, estimation becomes more challenging, and more labeled samples are required for stable convergence. Nevertheless, OTROPE still achieves clear improvements over the baselines, demonstrating that the OT-based correction remains effective even in this difficult regime. Additional empirical checks are provided in the appendix: Appendix~\ref{app:prop_verification} verifies Proposition~\ref{prop:dr_bad}, and Appendix~\ref{app:variance_shift} reports an additional simulation under variance shift rather than mean shift.

\subsection{Real Data Applications on LLM Evaluations}
\label{subsec:real}

\textbf{Datasets and models.} 
\textsc{RewardBench~2}~\citep{malik2026rewardbench} is a multi-skill reward modeling benchmark containing model-generated responses and human preference annotations. \textsc{Chatbot Arena}~\citep{zheng2023judging} provides large-scale human preference data from pairwise comparisons of LLM responses. We adapt both datasets to our evaluation setting as follows. We select two LLMs, or two pools of LLMs, to serve as the behavior and target policies, respectively, and construct an additional pool of LLMs as opponents. For each candidate response generated by either the behavior or target policy, we form a pairwise comparison against a response generated by an opponent model. Human preferences define the binary label $g^\ast$, where $g^\ast=1$ indicates that the candidate response is preferred over the opponent response and $g^\ast=0$ otherwise. The policy value is then defined as the expected preference win rate against the opponent pool.
For each benchmark, we consider two target-policy settings. In \textsc{RewardBench~2}, we consider a pooled target policy consisting of Claude 3.5 Sonnet and human responses, and a single-model target policy given by Claude 3.5 Sonnet. In \textsc{Chatbot Arena}, we use GPT-4 and GPT-3.5-Turbo as the two target policies. The behavior-policy configurations and more data construction details are provided in Appendix~\ref{asec:data_cons}. 
All experiments are conducted in a black-box model setting, where token-level probabilities are unavailable. Empirical support for Assumptions~\ref{ass:bound_emb} and~\ref{ass:ball} over \textsc{RewardBench~2} and \textsc{Chatbot Arena} is provided in Appendices~\ref{app:support} and~\ref{app:ball}, respectively.

\textbf{Implementation.} 
We encode each prompt--response pair using bge-m3 embeddings~\citep{chen2024bge}, and for pairwise comparisons, concatenate the embeddings of the candidate and opponent responses. 
Theorem~\ref{the:ot_rate} shows that the convergence rate depends on both the embedding dimension $d$ and the intrinsic dimension $d_0$. This motivates the use of principal component analysis (PCA)~\citep{hotelling1933analysis,abdi2010principal} before solving the OT problem: reducing the embedding dimension can improve convergence and computational efficiency, while the low-intrinsic-dimensional structure suggests that the dominant components retain most of the relevant information.
Since token probability of responses are not unavailable in the black-box setting, we use a classifier-based (CLS) importance weight estimator~\citep{sugiyama2012density} to approximate the density ratios in IS and DR (also see Appendix~\ref{app:baseline}). Each setting is repeated for $10$ times.

\begin{table*}[!t]
\caption{
Mean absolute error of the estimated target policy value on
\textsc{RewardBench~2} and \textsc{Chatbot Arena}.
Standard deviations are reported in parentheses.
The best result in each column is highlighted in \textbf{bold}.
The Mixed policy is the mixture of Claude 3.5 Sonnet and human responses.
}
\label{tab:ot-chat}

\centering
\scriptsize
\setlength{\tabcolsep}{5pt}

\begin{tabular}{ll|cc|cc}
\toprule

\multirow{2}{*}{\textbf{Method}}
& \textbf{Benchmark}
& \multicolumn{2}{c|}{\textsc{RewardBench~2}}
& \multicolumn{2}{c}{\textsc{Chatbot Arena}} \\

\cmidrule{3-4}
\cmidrule{5-6}

& \textbf{Target Policy ($\pi$)}
& Mixed
& Claude 3.5 Sonnet
& GPT-4
& GPT-3.5-Turbo \\

\midrule


\greencell
\cellcolor{white}
& IS
& 0.263 {\scriptsize(0.006)}
& 0.273 {\scriptsize(0.008)}
& 0.414 {\scriptsize(0.002)}
& 0.272 {\scriptsize(0.001)} \\

\greencell
\cellcolor{white}
& DR
& 0.129 {\scriptsize(0.008)}
& 0.125 {\scriptsize(0.003)}
& 0.088 {\scriptsize(0.007)}
& 0.065 {\scriptsize(0.004)} \\

\yellowcell
\cellcolor{white}
& PPI
& 0.105 {\scriptsize(0.006)}
& 0.139 {\scriptsize(0.001)}
& 0.086 {\scriptsize(0.007)}
& 0.075 {\scriptsize(0.005)} \\

\yellowcell
\cellcolor{white}
& PPI++
& 0.116 {\scriptsize(0.003)}
& 0.138 {\scriptsize(0.001)}
& 0.301 {\scriptsize(0.004)}
& 0.196 {\scriptsize(0.003)} \\

\redcell
\cellcolor{white}
\multirow{-5}{*}{$\widehat{g}=$ \textit{Mixture-of-Experts}}
& \textbf{OTROPE}
& \best{0.071} {\scriptsize(0.046)}
& \best{0.103} {\scriptsize(0.077)}
& \best{0.049} {\scriptsize(0.018)}
& \best{0.010} {\scriptsize(0.007)} \\

\midrule


\greencell
\cellcolor{white}
& IS
& 0.263 {\scriptsize(0.006)}
& 0.273 {\scriptsize(0.008)}
& 0.414 {\scriptsize(0.002)}
& 0.272 {\scriptsize(0.001)} \\

\greencell
\cellcolor{white}
& DR
& 0.123 {\scriptsize(0.006)}
& 0.125 {\scriptsize(0.007)}
& 0.130 {\scriptsize(0.001)}
& 0.093 {\scriptsize(0.002)} \\

\yellowcell
\cellcolor{white}
& PPI
& 0.097 {\scriptsize(0.006)}
& 0.122 {\scriptsize(0.007)}
& 0.124 {\scriptsize(0.002)}
& 0.108 {\scriptsize(0.002)} \\

\yellowcell
\cellcolor{white}
& PPI++
& 0.118 {\scriptsize(0.003)}
& 0.135 {\scriptsize(0.002)}
& 0.365 {\scriptsize(0.001)}
& 0.234 {\scriptsize(0.001)} \\

\redcell
\cellcolor{white}
\multirow{-5}{*}{$\widehat{g}=$ \textit{Majority Voting}}
& \textbf{OTROPE}
& \best{0.066} {\scriptsize(0.059)}
& \best{0.088} {\scriptsize(0.083)}
& \best{0.029} {\scriptsize(0.014)}
& \best{0.013} {\scriptsize(0.012)} \\

\midrule


\bluecell
\cellcolor{white}
\textit{Stronger LLM-as-a-Judge}
& DeepSeek-V3.1
& 0.164
& 0.159
& 0.153
& 0.228 \\

\bottomrule
\end{tabular}
\end{table*}

\textbf{Results.}
As shown in Table~\ref{tab:ot-chat}, OTROPE consistently achieves the lowest MAE across all datasets under different target policies  and proxy evaluators. 
The larger standard deviations observed on \textsc{RewardBench~2} arise because the behavior and target policies are constructed as pools of LLMs rather than from a single LLM, resulting in hybrid behavior and target distributions. 
Baseline estimators may achieve lower variance but suffer from persistent large bias, whereas our method reduces this bias and yields more accurate estimates. The variance of OTROPE is also affected by the regularization term $\lambda$ introduced later in Section~\ref{subsec:ablation}, which controls a trade-off between matching stability and accuracy.
Furthermore, OTROPE with small LLMs (all with <15B parameters, detailed in Appendix~\ref{asec:llms}) consistently outperforms all baselines and even surpasses the stronger standalone judge, DeepSeek-V3.1~\citep{liu2024deepseek} (with 671B parameters). This shows that \textit{OT-based correction} can enable \textit{ensembles of weaker evaluators} to match or \textit{exceed much stronger models}.

\begin{table*}[!t]
\caption{Ablation study under the Mixed policy setting on \textsc{RewardBench~2}, showing the effects of the correction term and proxy predictors across different labeled sample sizes $n$.
DM with majority voting is invariant to $n$ as as it does not use labeled samples; OT is identical across aggregation strategies as it does not use proxy predictions.} 
\label{tab:ablation}
\centering
\scriptsize
\setlength{\tabcolsep}{5pt}
\begin{tabular}{l |ccc |ccc}
\toprule
\multirow{2}{*}[-0.3ex]{\textbf{Estimator}} & 
\multicolumn{3}{c|}{Mixture-of-Experts} & \multicolumn{3}{c}{Majority Voting} \\
\cmidrule{2-4} \cmidrule{5-7} 
& $n=800$ & $n=1200$ & $n=1600$ & $n=800$ & $n=1200$ & $n=1600$  \\
\midrule

DM
& 0.252 {\scriptsize(0.009)}
& 0.252 {\scriptsize(0.007)}
& 0.256 {\scriptsize(0.006)}
& 0.296 {\scriptsize(0.000)} 
& 0.296 {\scriptsize(0.000)} 
& 0.296 {\scriptsize(0.000)} \\

OT
& 0.272 {\scriptsize(0.073)}
& 0.236 {\scriptsize(0.085)}
& 0.258 {\scriptsize(0.050)}
& 0.272 {\scriptsize(0.073)} 
& 0.236 {\scriptsize(0.085)} 
& 0.258 {\scriptsize(0.050)} \\

\redcell \textbf{OTROPE} 
& \best{0.144} {\scriptsize(0.048)}
& \best{0.114} {\scriptsize(0.036)}
& \best{0.071} {\scriptsize(0.046)}
& \best{0.099} {\scriptsize(0.083)}
& \best{0.117} {\scriptsize(0.067)}
& \best{0.066} {\scriptsize(0.059)}\\

\bottomrule
\end{tabular}
\end{table*}

\subsection{Additional Results}
\label{subsec:ablation}
\textbf{Ablation study. }We consider two ablations. First, we remove the correction term and use the DM estimator based only on proxy evaluations. Second, we remove the proxy evaluations and apply OT-based reweighting directly to the ground-truth evaluations, i.e.,
$\widehat{V}_{\mathrm{OT}}(\pi)= (1/n) \sum_{i=1}^n w_i^\ast \, g^\ast(x_i, y_i)$.
We conduct the ablation study on the Mixed policy in \textsc{RewardBench~2} and summarize the results in Table~\ref{tab:ablation}. 
The relatively large errors of the DM estimator indicate that proxy evaluations alone can be unreliable. 
The superior performance of OTROPE highlights the importance of the correction term, which mitigates proxy misspecification and improves robustness.
The OT estimator also performs worse than OTROPE because, when the proxy evaluator captures the dominant variation in the ground-truth evaluation, the residual $g^\ast-\widehat g$ is expected to vary more smoothly than $g^\ast$ itself.

\begin{wraptable}{r}{0.48\textwidth}
\caption{Stable performance with different $\lambda$ and sample size $n$.}
\label{tab:ot_lambda_sensitivity}
\centering
\scriptsize
\setlength{\tabcolsep}{5pt}
\begin{tabular}{l ccc }
\toprule
\multirow{2}{*}[-0.7ex]{$\lambda$} & 
\multicolumn{3}{c}{$n$} \\
\cmidrule{2-4}
& 800 & 1200 & 1600 \\
\midrule

0.05
& 0.080 {\scriptsize(0.025)}
& 0.058 {\scriptsize(0.028)}
& 0.054 {\scriptsize(0.015)}
\\

0.08
& 0.116 {\scriptsize(0.073)}
& 0.053 {\scriptsize(0.061)}
& 0.042 {\scriptsize(0.024)}
\\

0.10
& 0.086 {\scriptsize(0.060)}
& 0.062 {\scriptsize(0.025)}
& 0.049 {\scriptsize(0.018)}
\\

\bottomrule
\end{tabular}
\vspace{-1em}
\end{wraptable}

\textbf{Sensitivity analysis of $\lambda$. }In our method, $\lambda$ is the entropic regularization parameter in the Sinkhorn iterations. 
It controls the trade-off between matching accuracy and computational stability; see Appendix~\ref{app:ot_optimize} for optimization details. 
Smaller $\lambda$ gives more discriminative OT matching but is more computationally expensive, while larger $\lambda$ improves efficiency but may reduce accuracy by smoothing pairwise cost differences. 
We conduct a sensitivity analysis on \textsc{Chatbot Arena} using GPT-4 as the target policy with MoE aggregation. The results in Table~\ref{tab:ot_lambda_sensitivity} show that stable performance across different sample sizes for $\lambda \in [0.05,0.10]$.
Smaller values of $\lambda$ are preferred for more heterogeneous distributions, as they preserve finer pairwise cost structure and yield more accurate matching. 
For larger sample sizes, we increase $\lambda$ within the empirically stable range to improve computational efficiency.

\textbf{Sensitivity to the embedding representation.}
To assess whether OTROPE depends on a specific high-quality embedding model, we replace bge-m3 with a simple TF--IDF~\citep{salton1988term} bag-of-words vectorizer using 1024 features and up to 2-grams, while keeping the rest of the pipeline unchanged. As shown in Table~\ref{tab:tfidf_ablation}, OTROPE remains the best-performing estimator in three out of the four aggregation--target-policy combinations. These results suggest that OTROPE is reasonably robust to the choice of representation and remains effective even with a much simpler lexical embedding. Moreover, its performance is generally comparable to that obtained with bge-m3, as shown in Table~\ref{tab:ot-chat}.

\begin{table}[h]
\centering
\begin{minipage}[t]{0.56\textwidth}
\centering
\caption{Mean absolute error of the estimated target policy value on \textsc{Chatbot Arena} using TF--IDF embeddings.}
\label{tab:tfidf_ablation}
\scriptsize
\setlength{\tabcolsep}{3pt}
\begin{tabular}{lcc|cc}
\toprule
& \multicolumn{2}{c|}{Mixture-of-Experts} 
& \multicolumn{2}{c}{Majority Voting} \\
\cmidrule(lr){2-3} \cmidrule(lr){4-5}
Estimator & GPT-4 & GPT-3.5 Turbo & GPT-4 & GPT-3.5 Turbo \\
\midrule
IS     
& 0.438 (0.004) 
& 0.322 (0.003) 
& 0.438 (0.003) 
& 0.322 (0.003) \\

DR     
& 0.023 (0.011) 
& \textbf{0.006 (0.005)} 
& \textbf{0.105 (0.004)} 
& 0.036 (0.003) \\

PPI    
& 0.028 (0.011) 
& 0.023 (0.007) 
& 0.124 (0.002) 
& 0.108 (0.002) \\

PPI++  
& 0.276 (0.004) 
& 0.174 (0.003) 
& 0.365 (0.001) 
& 0.234 (0.001) \\

\redcell
OTROPE 
& \textbf{0.013 (0.009)} 
& \textbf{0.006 (0.004)} 
& 0.114 (0.009) 
& \textbf{0.007 (0.004)} \\
\bottomrule
\end{tabular}
\end{minipage}
\hfill
\begin{minipage}[t]{0.40\textwidth}
\centering
\caption{Mean absolute error on \textsc{Chatbot Arena} under GPT-3.5-Turbo, with $n=800$ labeled and $N=1080$ unlabeled samples.}
\label{tab:real_data_label_scarce}
\scriptsize
\setlength{\tabcolsep}{4pt}
\begin{tabular}{lcc}
\toprule
Estimator & Mixture-of-Experts & Majority Voting \\
\midrule
IS     
& 0.287 (0.011) 
& 0.287 (0.011) \\

DR     
& 0.077 (0.015) 
& 0.091 (0.016) \\

PPI    
& 0.092 (0.017) 
& 0.113 (0.016) \\

PPI++  
& 0.187 (0.016) 
& 0.234 (0.012) \\

\redcell
OTROPE 
& \textbf{0.045 (0.037)} 
& \textbf{0.083 (0.042)} \\
\bottomrule
\end{tabular}
\end{minipage}
\end{table}

\textbf{Evaluation in the label-scarce regime. } To evaluate OTROPE in a real-data setting closer to the label-scarce regime considered in our simulations, we conduct the \textsc{Chatbot Arena} experiment with GPT-3.5-Turbo as the target policy, using $n=800$ labeled behavior-policy samples and $N=1080$ unlabeled target-policy samples. All other experimental configurations and hyperparameters follow those described above. The results are reported in Table~\ref{tab:real_data_label_scarce}.
OTROPE achieves the lowest MAE under both aggregation schemes when $n<N$, indicating that the proposed OT-based correction remains effective in this label-scarce real-data setting.

\textbf{Sensitivity to PCA dimension and transport cost.}
We further examine the robustness of OTROPE to two implementation choices: the PCA truncation dimension and the transport cost function. The corresponding sensitivity analyses are provided in Appendices~\ref{app:pca_sensitivity} and~\ref{app:cost_sensitivity}, respectively.

\section{Conclusion and Limitation}
In this paper, we study LLM evaluation from an OPE perspective. We propose OTROPE, which uses optimal transport to reweight labeled residuals in a semantic embedding space. We prove consistency and convergence rates, and show that OTROPE consistently outperforms baselines on synthetic and real LLM evaluation tasks. Despite these advantages, OTROPE has several limitations. First, OT-based weighting can be computationally expensive; we use PCA for efficiency, but it may lose fine-grained information in dense embedding spaces. Second, our current formulation focuses on static input--output evaluation. Extending it to sequential evaluation settings, such as multi-turn interactions that can be modeled as bandits or Markov decision processes, is left for future work.

\begin{ack}
This work was supported in part by the National Science Foundation under Grant No. DMS-2401271 and by Thinking Machine Lab through the Tinker Research Grant.
We thank the area chair and anonymous reviewers for their helpful comments.
\end{ack}

\bibliography{citation}

@article{angelopoulos2023prediction,
  title={Prediction-powered inference},
  author={Angelopoulos, Anastasios N and Bates, Stephen and Fannjiang, Clara and Jordan, Michael I and Zrnic, Tijana},
  journal={Science},
  volume={382},
  number={6671},
  pages={669--674},
  year={2023},
  publisher={American Association for the Advancement of Science}
}

@article{angelopoulos2023ppi++,
  title={Ppi++: Efficient prediction-powered inference},
  author={Angelopoulos, Anastasios N and Duchi, John C and Zrnic, Tijana},
  journal={arXiv preprint arXiv:2311.01453},
  year={2023}
}

@inproceedings{dudik2011doubly,
  title={Doubly robust policy evaluation and learning},
  author={Dud{\'\i}k, Miroslav and Langford, John and Li, Lihong},
  booktitle={Proceedings of the 28th International Conference on International Conference on Machine Learning},
  pages={1097--1104},
  year={2011}
}

@inproceedings{cuturi2013sinkhorn,
 author = {Cuturi, Marco},
 booktitle = {Advances in Neural Information Processing Systems},
 editor = {C.J. Burges and L. Bottou and M. Welling and Z. Ghahramani and K.Q. Weinberger},
 pages = {},
 publisher = {Curran Associates, Inc.},
 title = {Sinkhorn Distances: Lightspeed Computation of Optimal Transport},
 url = {https://proceedings.neurips.cc/paper_files/paper/2013/file/af21d0c97db2e27e13572cbf59eb343d-Paper.pdf},
 volume = {26},
 year = {2013}
}

@article{fournier2015rate,
  title={On the rate of convergence in Wasserstein distance of the empirical measure},
  author={Fournier, Nicolas and Guillin, Arnaud},
  journal={Probability theory and related fields},
  volume={162},
  number={3},
  pages={707--738},
  year={2015},
  publisher={Springer}
}

@book{sarndal2003model,
  title={Model assisted survey sampling},
  author={S{\"a}rndal, Carl-Erik and Swensson, Bengt and Wretman, Jan},
  year={2003},
  publisher={Springer Science \& Business Media}
}

@article{cassel1976some,
  title={Some results on generalized difference estimation and generalized regression estimation for finite populations},
  author={Cassel, Claes M and S{\"a}rndal, Carl E and Wretman, Jan H},
  journal={Biometrika},
  volume={63},
  number={3},
  pages={615--620},
  year={1976},
  publisher={Oxford University Press}
}

@article{fisch2024stratified,
  title={Stratified prediction-powered inference for effective hybrid evaluation of language models},
  author={Fisch, Adam and Maynez, Joshua and Hofer, R and Dhingra, Bhuwan and Globerson, Amir and Cohen, William W},
  journal={Advances in Neural Information Processing Systems},
  volume={37},
  pages={111489--111514},
  year={2024}
}

@article{boyeau2024autoeval,
  title={Autoeval done right: Using synthetic data for model evaluation},
  author={Boyeau, Pierre and Angelopoulos, Anastasios N and Yosef, Nir and Malik, Jitendra and Jordan, Michael I},
  journal={arXiv preprint arXiv:2403.07008},
  year={2024}
}

@inproceedings{saad2024ares,
  title={Ares: An automated evaluation framework for retrieval-augmented generation systems},
  author={Saad-Falcon, Jon and Khattab, Omar and Potts, Christopher and Zaharia, Matei},
  booktitle={Proceedings of the 2024 Conference of the North American Chapter of the Association for Computational Linguistics: Human Language Technologies (Volume 1: Long Papers)},
  pages={338--354},
  year={2024}
}

@article{cowen2026multiple,
  title={Multiple-Prediction-Powered Inference},
  author={Cowen-Breen, Charlie and Agarwal, Alekh and Bates, Stephen and Cohen, William W and Eisenstein, Jacob and Globerson, Amir and Fisch, Adam},
  journal={arXiv preprint arXiv:2603.27414},
  year={2026}
}

@article{datta2025prediction,
  title={Prediction-powered inference with inverse probability weighting},
  author={Datta, Jyotishka and Polson, Nicholas G},
  journal={arXiv preprint arXiv:2508.10149},
  year={2025}
}

@inproceedings{xu2026doubly,
title={Doubly Robust Alignment for Large Language Models},
author={Erhan Xu and Kai Ye and Hongyi Zhou and Luhan Zhu and Francesco Quinzan and Chengchun Shi},
booktitle={The Thirty-ninth Annual Conference on Neural Information Processing Systems},
year={2026},
url={https://openreview.net/forum?id=HvklLrtyxK}
}

@article{precup2000eligibility,
  title={Eligibility traces for off-policy policy evaluation},
  author={Precup, Doina and Sutton, Richard S and Singh, Satinder},
  year={2000}
}

@book{sutton1998reinforcement,
  title={Reinforcement learning: An introduction},
  author={Sutton, Richard S and Barto, Andrew G and others},
  volume={1},
  number={1},
  year={1998},
  publisher={MIT press Cambridge}
}

@inproceedings{li2010contextual,
  title={A contextual-bandit approach to personalized news article recommendation},
  author={Li, Lihong and Chu, Wei and Langford, John and Schapire, Robert E},
  booktitle={Proceedings of the 19th international conference on World wide web},
  pages={661--670},
  year={2010}
}

@inproceedings{jiang2016doubly,
  title={Doubly robust off-policy value evaluation for reinforcement learning},
  author={Jiang, Nan and Li, Lihong},
  booktitle={International conference on machine learning},
  pages={652--661},
  year={2016},
  organization={PMLR}
}

@article{liu2018breaking,
  title={Breaking the curse of horizon: Infinite-horizon off-policy estimation},
  author={Liu, Qiang and Li, Lihong and Tang, Ziyang and Zhou, Dengyong},
  journal={Advances in neural information processing systems},
  volume={31},
  year={2018}
}

@inproceedings{thomas2016data,
  title={Data-efficient off-policy policy evaluation for reinforcement learning},
  author={Thomas, Philip and Brunskill, Emma},
  booktitle={International conference on machine learning},
  pages={2139--2148},
  year={2016},
  organization={PMLR}
}

@book{villani2009optimal,
  title={Optimal transport: old and new},
  author={Villani, C{\'e}dric and others},
  volume={338},
  year={2009},
  publisher={Springer}
}

@article{vaserstein1969markov,
  title={Markov processes over denumerable products of spaces, describing large systems of automata},
  author={Vaserstein, Leonid Nisonovich},
  journal={Problemy Peredachi Informatsii},
  volume={5},
  number={3},
  pages={64--72},
  year={1969},
  publisher={Russian Academy of Sciences, Branch of Informatics, Computer Equipment and~…}
}

@article{rubner2000earth,
  title={The earth mover's distance as a metric for image retrieval},
  author={Rubner, Yossi and Tomasi, Carlo and Guibas, Leonidas J},
  journal={International journal of computer vision},
  volume={40},
  number={2},
  pages={99--121},
  year={2000},
  publisher={Springer}
}

@article{melnyk2024distributional,
  title={Distributional preference alignment of llms via optimal transport},
  author={Melnyk, Igor and Mroueh, Youssef and Belgodere, Brian and Rigotti, Mattia and Nitsure, Apoorva and Yurochkin, Mikhail and Greenewald, Kristjan and Navratil, Jiri and Ross, Jarret},
  journal={Advances in Neural Information Processing Systems},
  volume={37},
  pages={104412--104442},
  year={2024}
}

@article{manole2024sharp,
  title={Sharp convergence rates for empirical optimal transport with smooth costs},
  author={Manole, Tudor and Niles-Weed, Jonathan},
  journal={The Annals of Applied Probability},
  volume={34},
  number={1B},
  pages={1108--1135},
  year={2024},
  publisher={Institute of Mathematical Statistics}
}

@article{gangbo1996geometry,
  title={The geometry of optimal transportation},
  author={Gangbo, Wilfrid and McCann, Robert J},
  year={1996}
}

@inproceedings{hundrieser2024empirical,
  title={Empirical optimal transport between different measures adapts to lower complexity},
  author={Hundrieser, Shayan and Staudt, Thomas and Munk, Axel},
  booktitle={Annales de l'Institut Henri Poincare (B) Probabilites et statistiques},
  volume={60},
  number={2},
  pages={824--846},
  year={2024},
  organization={Institut Henri Poincar{\'e}}
}

@article{weed2019sharp,
  title={Sharp asymptotic and finite-sample rates of convergence of empirical measures in Wasserstein distance},
  author={Weed, Jonathan and Bach, Francis},
  journal={Bernoulli},
  volume={25},
  number={4A},
  pages={2620--2648},
  year={2019},
  publisher={JSTOR}
}

@inproceedings{malik2026rewardbench,
title={RewardBench 2: Advancing Reward Model Evaluation},
author={Saumya Malik and Valentina Pyatkin and Sander Land and Jacob Morrison and Noah A. Smith and Hannaneh Hajishirzi and Nathan Lambert},
booktitle={The Fourteenth International Conference on Learning Representations},
year={2026},
url={https://openreview.net/forum?id=fb0G86Dewb}
}

@book{sugiyama2012density,
  title={Density ratio estimation in machine learning},
  author={Sugiyama, Masashi and Suzuki, Taiji and Kanamori, Takafumi},
  year={2012},
  publisher={Cambridge University Press}
}

@article{chen2024bge,
  title={Bge m3-embedding: Multi-lingual, multi-functionality, multi-granularity text embeddings through self-knowledge distillation},
  author={Chen, Jianlv and Xiao, Shitao and Zhang, Peitian and Luo, Kun and Lian, Defu and Liu, Zheng},
  journal={arXiv preprint arXiv:2402.03216},
  volume={4},
  number={5},
  year={2024}
}

@article{zheng2023judging,
  title={Judging llm-as-a-judge with mt-bench and chatbot arena},
  author={Zheng, Lianmin and Chiang, Wei-Lin and Sheng, Ying and Zhuang, Siyuan and Wu, Zhanghao and Zhuang, Yonghao and Lin, Zi and Li, Zhuohan and Li, Dacheng and Xing, Eric and others},
  journal={Advances in neural information processing systems},
  volume={36},
  pages={46595--46623},
  year={2023}
}

@inproceedings{chen2024humans,
  title={Humans or LLMs as the judge? a study on judgement bias},
  author={Chen, Guiming Hardy and Chen, Shunian and Liu, Ziche and Jiang, Feng and Wang, Benyou},
  booktitle={Proceedings of the 2024 Conference on Empirical Methods in Natural Language Processing},
  pages={8301--8327},
  year={2024}
}

@article{ye2024justice,
  title={Justice or prejudice? quantifying biases in llm-as-a-judge},
  author={Ye, Jiayi and Wang, Yanbo and Huang, Yue and Chen, Dongping and Zhang, Qihui and Moniz, Nuno and Gao, Tian and Geyer, Werner and Huang, Chao and Chen, Pin-Yu and others},
  journal={arXiv preprint arXiv:2410.02736},
  year={2024}
}

@article{yang2023gpteval,
  title={Gpteval: Nlg evaluation using gpt-4 with better human alignment},
  author={Yang, Liu and Dan, Iter and Xu, Yichong and Shuohang, Wang and Xu, Ruochen and Chenguang, Zhu},
  journal={arXiv preprint arXiv: 2303.16634},
  year={2023}
}

@inproceedings{fu2024gptscore,
  title={Gptscore: Evaluate as you desire},
  author={Fu, Jinlan and Ng, See Kiong and Jiang, Zhengbao and Liu, Pengfei},
  booktitle={Proceedings of the 2024 Conference of the North American Chapter of the Association for Computational Linguistics: Human Language Technologies (Volume 1: Long Papers)},
  pages={6556--6576},
  year={2024}
}

@article{courty2016optimal,
  title={Optimal transport for domain adaptation},
  author={Courty, Nicolas and Flamary, R{\'e}mi and Tuia, Devis and Rakotomamonjy, Alain},
  journal={IEEE transactions on pattern analysis and machine intelligence},
  volume={39},
  number={9},
  pages={1853--1865},
  year={2016},
  publisher={IEEE}
}

@inproceedings{redko2017theoretical,
  title={Theoretical analysis of domain adaptation with optimal transport},
  author={Redko, Ievgen and Habrard, Amaury and Sebban, Marc},
  booktitle={Joint European Conference on Machine Learning and Knowledge Discovery in Databases},
  pages={737--753},
  year={2017},
  organization={Springer}
}

@article{schroeder2024can,
  title={Can you trust llm judgments? reliability of llm-as-a-judge},
  author={Schroeder, Kayla and Wood-Doughty, Zach},
  journal={arXiv preprint arXiv:2412.12509},
  year={2024}
}

@article{reygner2022reweighting,
  title={Reweighting samples under covariate shift using a Wasserstein distance criterion},
  author={Reygner, Julien and Touboul, Adrien},
  journal={Electronic Journal of Statistics},
  volume={16},
  number={1},
  pages={3278--3314},
  year={2022}
}

@article{shen2023wasserstein,
  title={Wasserstein distributionally robust policy evaluation and learning for contextual bandits},
  author={Shen, Yi and Xu, Pan and Zavlanos, Michael M},
  journal={arXiv preprint arXiv:2309.08748},
  year={2023}
}

@article{liu2024deepseek,
  title={Deepseek-v3 technical report},
  author={Liu, Aixin and Feng, Bei and Xue, Bing and Wang, Bingxuan and Wu, Bochao and Lu, Chengda and Zhao, Chenggang and Deng, Chengqi and Zhang, Chenyu and Ruan, Chong and others},
  journal={arXiv preprint arXiv:2412.19437},
  year={2024}
}

@article{huang2025pluralistic,
  title={Pluralistic Off-policy Evaluation and Alignment},
  author={Huang, Chengkai and Wu, Junda and Xie, Zhouhang and Xia, Yu and Wang, Rui and Yu, Tong and Mitra, Subrata and McAuley, Julian and Yao, Lina},
  journal={arXiv preprint arXiv:2509.19333},
  year={2025}
}

@article{bhargava2024off,
  title={Off-policy evaluation from logged human feedback},
  author={Bhargava, Aniruddha and Jain, Lalit and Kveton, Branislav and Liu, Ge and Mukherjee, Subhojyoti},
  journal={arXiv preprint arXiv:2406.10030},
  year={2024}
}

@article{wu2024ocean,
  title={Ocean: Offline chain-of-thought evaluation and alignment in large language models},
  author={Wu, Junda and Li, Xintong and Wang, Ruoyu and Xia, Yu and Xiong, Yuxin and Wang, Jianing and Yu, Tong and Chen, Xiang and Kveton, Branislav and Yao, Lina and others},
  journal={arXiv preprint arXiv:2410.23703},
  year={2024}
}

@book{rachev1998mass,
  title={Mass Transportation Problems: Volume I: Theory},
  author={Rachev, Svetlozar T and R{\"u}schendorf, Ludger},
  year={1998},
  publisher={Springer}
}

@article{kantorovich1958space,
  author = {Kantorovich, Leonid and  Rubinstein, Gennady S.},
  title = {On a space of totally additive functions},
  journal = {Vestnik Leningrad. Univ},
  pages = {52--59},
  volume = {13},
  year = {1958}
}

@article{rakotomamonjy2022optimal,
  title={Optimal transport for conditional domain matching and label shift},
  author={Rakotomamonjy, Alain and Flamary, R{\'e}mi and Gasso, Gilles and Alaya, M El and Berar, Maxime and Courty, Nicolas},
  journal={Machine Learning},
  volume={111},
  number={5},
  pages={1651--1670},
  year={2022},
  publisher={Springer}
}

@book{peyre2019computational,
  title={Computational optimal transport: With applications to data science},
  author={Peyr{\'e}, Gabriel and Cuturi, Marco},
  year={2019},
  publisher={Now Foundations and Trends}
}

@inproceedings{turrisi2022multi,
  title={Multi-source domain adaptation via weighted joint distributions optimal transport},
  author={Turrisi, Rosanna and Flamary, R{\'e}mi and Rakotomamonjy, Alain and Pontil, Massimiliano},
  booktitle={Uncertainty in artificial intelligence},
  pages={1970--1980},
  year={2022},
  organization={PMLR}
}

@article{abdi2010principal,
  title={Principal component analysis},
  author={Abdi, Herv{\'e} and Williams, Lynne J},
  journal={Wiley interdisciplinary reviews: computational statistics},
  volume={2},
  number={4},
  pages={433--459},
  year={2010},
  publisher={Wiley Online Library}
}

@article{hotelling1933analysis,
  title={Analysis of a complex of statistical variables into principal components.},
  author={Hotelling, Harold},
  journal={Journal of educational psychology},
  volume={24},
  number={6},
  pages={417},
  year={1933},
  publisher={Warwick \& York}
}

@article{salton1988term,
  title={Term-weighting approaches in automatic text retrieval},
  author={Salton, Gerard and Buckley, Christopher},
  journal={Information processing \& management},
  volume={24},
  number={5},
  pages={513--523},
  year={1988},
  publisher={Elsevier}
}

@article{cai2021deep,
  title={Deep jump learning for off-policy evaluation in continuous treatment settings},
  author={Cai, Hengrui and Shi, Chengchun and Song, Rui and Lu, Wenbin},
  journal={Advances in Neural Information Processing Systems},
  volume={34},
  pages={15285--15300},
  year={2021}
}

@inproceedings{cai2020validation,
  title={On validation and planning of an optimal decision rule with application in healthcare studies},
  author={Cai, Hengrui and Lu, Wenbin and Song, Rui},
  booktitle={International Conference on Machine Learning},
  pages={1262--1270},
  year={2020},
  organization={PMLR}
}

@article{shen2024doubly,
  title={Doubly robust interval estimation for optimal policy evaluation in online learning},
  author={Shen, Ye and Cai, Hengrui and Song, Rui},
  journal={Journal of the American Statistical Association},
  volume={119},
  number={548},
  pages={2811--2821},
  year={2024},
  publisher={Taylor \& Francis}
}
\bibliographystyle{apalike}

\newpage
\appendix

\counterwithin{equation}{section}
\counterwithin{table}{section}
\counterwithin{figure}{section}

\section*{Appendix}
The appendix contains additional related work, theoretical analysis, experimental details, and supplementary results. Appendix~\ref{app:related_work} provides a more comprehensive review of OPE, PPI, and OT. Appendix~\ref{app:theoretical_analysis} presents theoretical results on the failure modes of PPI and DR, as well as the consistency, convergence rate, and double robustness of the proposed OTROPE estimator. Appendix~\ref{app:imple} provides more details on the training procedure and implementation. 
Appendix~\ref{app:exp} describes the data generation process, baseline estimators, LLMs used for evaluation, and hyperparameter choices. Finally, Appendix~\ref{app:more_results} reports additional experimental results.

\section{Related Work}
\label{app:related_work}

\paragraph{Off-policy Evaluation.}
Off-policy evaluation (OPE)~\citep{sutton1998reinforcement} aims to estimate the value of a target policy prior to deployment, using historical data collected under a different behavior policy. This strategy is particularly important in domains where online evaluation is expensive, risky, or potentially harmful~\citep{shen2024doubly}. Standard approaches to OPE include importance sampling~\citep{precup2000eligibility,liu2018breaking}, the direct method (DM)~\citep{li2010contextual}, and doubly robust (DR) estimators~\citep{dudik2011doubly,jiang2016doubly,thomas2016data,cai2020validation,cai2021deep}. Several recent works have adapted OPE to LLMs, but they address problems different from ours. \citet{bhargava2024off} study OPE from logged human feedback, with a focus on evaluating preference-ranking quality rather than raw generated responses. \citet{wu2024ocean,huang2025pluralistic,xu2026doubly} use OPE-inspired estimators primarily for alignment training, where the goal is to construct optimization objectives rather than to estimate the evaluation score of a target model. In contrast, we focus on response-level LLM evaluation under model shift, where the evaluation objects are high-dimensional text responses whose distributions shift across models.

\paragraph{Prediction-Powered Inference.} Prediction-Powered Inference (PPI)~\citep{angelopoulos2023prediction} is a class of semi-supervised methods that uses black-box model predictions as proxy variables to improve the efficiency and validity of classical statistical inference, closely related to difference estimators~\citep{cassel1976some,sarndal2003model}. Several extensions have been proposed within this framework. PPI++~\citep{angelopoulos2023ppi++}, for example, introduces a weighting parameter on autorater predictions to improve efficiency, while StratPPI~\citep{fisch2024stratified} considers settings where autorater performance varies across subpopulations or conditional distributions of the target data. PPI-type methods have also been applied to LLM evaluation~\citep{boyeau2024autoeval,saad2024ares,fisch2024stratified,cowen2026multiple}. However, these approaches typically assume that labeled and unlabeled responses are generated by the same model, and thus do not address settings where new models are evaluated using labeled historical responses from earlier models. Some prior works, including the original PPI literature, has considered distribution shift~\citep{angelopoulos2023prediction,datta2025prediction}, but only utilize raw importance weight.

\paragraph{Optimal Transport.} 
Optimal Transport (OT)~\citep{rachev1998mass,villani2009optimal} provides a mathematical framework for comparing two probability distributions by seeking the most cost-efficient way to transport mass from one distribution to another, with transportation costs specified through an underlying ground metric. The resulting Wasserstein distance~\citep{vaserstein1969markov}, also referred to as the Earth Mover's Distance~\citep{rubner2000earth}, or the Kantorovich--Rubinstein distance~\citep{kantorovich1958space}, has played a central role in probability and statistics, and has increasingly become a powerful tool in modern machine learning and data science~\citep{peyre2019computational}. Recently, OT has been increasingly used to address distributional discrepancies in machine learning, including covariate shift~\citep{reygner2022reweighting} and domain adaptation~\citep{courty2016optimal,redko2017theoretical,rakotomamonjy2022optimal,turrisi2022multi}. Although these works also leverage OT-based discrepancies or reweighting strategies, their objectives differ substantially from ours: they primarily focus on distribution alignment for supervised prediction, whereas our goal is OPE under distribution shift for LLMs.

\section{Theoretical Analysis}
\label{app:theoretical_analysis}

\subsection{Proof of Proposition~\ref{prop:dr_bad}}
\label{app:proof_prop_bad}
\begin{proof}
Let $T_n=\frac1n\sum_{i=1}^n \rho(z_i)$ and $R_n=\frac1n\sum_{i=1}^n \rho(z_i)r(z_i)$. 
Then $\widehat V_{\rm DR}(\pi)=\widehat V_{\rm DM}(\pi)+R_n$.
\paragraph{Bias lower bound.}
Consider the event $\mathcal E_1=\{T_n\le a_n\}$.
By assumption, \(\mathbb P(\mathcal E_1)\ge 1-d_n\). On this event $\mathcal E_1$,
using \(\rho(z_i)\ge 0\) and \(|r(z_i)|\le C_r\), we have
\[
|R_n| = \left|\frac1n\sum_{i=1}^n \rho(z_i)r(z_i) \right| \le
\frac1n\sum_{i=1}^n \left|\rho(z_i)r(z_i)\right|
\le \frac1n\sum_{i=1}^n \rho(z_i) C_r = C_r T_n \le C_r a_n .
\]
Therefore, with probability at least \(1-d_n\), $|R_n|\le C_r a_n$.

Denote the event $\mathcal E_2 = \left\{ \left|\widehat V_{\rm DM}(\pi) -V^\ast(\pi)\right|\ge b_n \right\}$.
By assumption, \(\mathbb P(\mathcal E_2)\ge 1-d_n\). 
There is
\begin{equation*}
\left| \widehat V_{\rm DR}(\pi)-V^\ast(\pi) \right| =
\left| \widehat V_{\rm DM}(\pi)+R_n-V^\ast(\pi) \right| \ge 
\left| \widehat V_{\rm DM}(\pi)-V^\ast(\pi) \right| - |R_n|,    
\end{equation*}
where the last inequality follows from the reverse triangle inequality.

On the event \(\mathcal E_1\cap \mathcal E_2\), 
\[
\left| \widehat V_{\rm DR}(\pi)-V^\ast(\pi) \right| \ge b_n-C_r a_n.
\]

Finally, by the union bound,
\[
\mathbb P(\mathcal E_1\cap \mathcal E_2) \ge
1-\mathbb P(\mathcal E_1^c)-\mathbb P(\mathcal E_2^c) \ge
1-d_n-d_n = 1-2d_n.
\]
Hence, with probability at least \(1-2d_n\),
\[
\left|
\widehat V_{\rm DR}(\pi)-V^\ast(\pi)
\right|
\ge
b_n-C_r a_n.
\]

\paragraph{Variance lower bound.}
We first lower bound the variance of \(R_n\). Since \(C_r > r(z)\ge c_r>0\),
\[
R_n = \frac1n\sum_{i=1}^n \rho(z_i)r(z_i) \ge c_r T_n,
\qquad
R_n \le C_r T_n.
\]
Therefore,
\[
\mathbb E[R_n^2] \ge c_r^2 \mathbb E[T_n^2],
\qquad
(\mathbb E[R_n])^2 \le C_r^2 (\mathbb E[T_n])^2.
\]
Since \(\mathbb E[T_n]=1\),
\[
\operatorname{Var}(R_n) = \mathbb E[R_n^2]-(\mathbb E[R_n])^2
\ge c_r^2 \mathbb E[T_n^2]-C_r^2.
\]
It remains to lower bound \(\mathbb E[T_n^2]\).
By assumption, \(\mathbb P(\mathcal{E}_1^c)\le d_n\).
We have
\[
1 = \mathbb E[T_n] = 
\mathbb E[T_n \mathbf{1}_{\mathcal{E}_1}] + \mathbb E[T_n \mathbf{1}_{\mathcal{E}_1^c}] \le
a_n + \mathbb E[T_n \mathbf{1}_{\mathcal{E}_1^c}],
\]
which implies
\[
\mathbb E[T_n \mathbf{1}_{\mathcal{E}_1^c}] \ge 1-a_n.
\]
By Cauchy--Schwarz inequality,
\[
\mathbb E[T_n \mathbf{1}_{\mathcal{E}_1^c}] \le
\{\mathbb E[T_n^2]\}^{1/2} \{\mathbb P(\mathcal{E}_1^c)\}^{1/2} \le
\{\mathbb E[T_n^2]\}^{1/2} d_n^{1/2}.
\]
Therefore,
\[
\mathbb E[T_n^2] \ge \frac{(1-a_n)^2}{d_n}.
\]
Combining the above bounds gives
\[
\operatorname{Var}(R_n) \ge c_r^2\frac{(1-a_n)^2}{d_n} - C_r^2.
\]
Finally,
\[
\operatorname{Var}\{\widehat V_{\rm DR}(\pi)\} =
\operatorname{Var}(\widehat V_{\rm DM}(\pi)) + \operatorname{Var}(R_n) + 2\operatorname{Cov}(\widehat V_{\rm DM}(\pi),R_n).
\]
Using
\[
|\operatorname{Cov}(\widehat V_{\rm DM},R_n)| \le \kappa \operatorname{Var}(R_n),
\]
we obtain
\[
\operatorname{Var}\{\widehat V_{\rm DR}(\pi)\} \ge (1-2\kappa)\operatorname{Var}(R_n).
\]
Substituting the lower bound for \(\operatorname{Var}(R_n)\) completes the proof.
\end{proof}

\subsection{Proof of Lemma~\ref{lemma}}
\label{app:proof_lemma}

The support of $e_{\#Q}$ is defined as: $e(z) \in \operatorname{supp}(e_{\#Q})$ if and only if $e_\#Q(B(e(z),\varepsilon)) > 0$ for all $\varepsilon > 0$, where $B(e(z),\varepsilon)$ denotes the open ball centered at $e(z)$ under the embedding space and where $e_\#Q(\cdot)$ denotes the probability mass in distribution $e_\#Q$.

\begin{proof}
Let $S\coloneq\operatorname{supp}(e_\#Q)$, which is compact by Assumption~\ref{ass:bound_emb}. For any $\varepsilon>0$.

Since $S$ is compact, there exist finitely many points $u_1,\dots,u_m\in S$ such that
\[
S\subseteq \bigcup_{k=1}^m B(u_k,\varepsilon),
\]
where $B(u,\varepsilon)$ denotes the ball centered at $u$ with radius $\varepsilon$ under the embedding-induced distance, i.e.,
\[
B(u,\varepsilon)
=
\{ u' : \|u^\prime - u\| < \varepsilon \}.
\]
Because $u_k\in \operatorname{supp}(e_\#Q)\subseteq \operatorname{supp}(e_\#P)$, each ball $B(u_k,\varepsilon)$ has strictly positive probability $e_\#P(B(u_k,\varepsilon))>0$, $k=1,\dots,m$. Define the event
\[
A_{n,k} \coloneq \bigcup_{i=1}^n \{ e(z_i) \in B(u_k,\varepsilon) \}.
\]
Then 
\[
\mathbb P(A_{n,k}^c)
=
\bigl(1-{P}(B(u_k,\varepsilon))\bigr)^n.
\]
Since $\mathbb P(B(u_k,\varepsilon))>0$, we have
\[
\sum_{n=1}^\infty \mathbb P(A_{n,k}^c)<\infty.
\]
for each $k$. Hence, by the Borel--Cantelli lemma,
\[
\mathbb P(\limsup_{n\to\infty} A_{n,k}^c\ )=0,
\]
This implies that $A_{n,k}^c$ occurs only finitely often almost surely, i.e.,
there exists a finite $N_k$ such that for all $n \ge N_k$, $A_{n,k}$ happens.
Let
\[
N_\varepsilon\coloneq\max_{1\le k\le m} N_k.
\]
Then for every $n\ge N_\varepsilon$, each ball $B(u_k,\varepsilon)$ contains at least one sample point among $e(z_1),\dots,e(z_n)$. For each $k=1,\dots,m$, choose one index $i_k \in\{1,\dots,n\}$ such that
\[
e(z_{i_k})\in B(u_k,\varepsilon).
\]
Next, construct a measurable partition $C_1,\dots,C_m$ of $S$ such that
\[
C_k\subseteq B(u_k,\varepsilon),\qquad \bigcup_{k=1}^m C_k=S.
\]
Define
\[
\mu_{n,\varepsilon}
\coloneq
\sum_{k=1}^m e_\#Q(C_k)\,\delta_{e(z_{i_k})},
\]
where $e_\#Q(C_k)$ denotes the probability mass of the set $C_k$ under the distribution $e_\#Q$.
Since each $i_k\in\{1,\dots,n\}$ and the coefficients are nonnegative and sum to one, we have $\mu_{n,\varepsilon}\in \mathcal M_n$.
Consider the coupling that sends all mass in $C_k$ to the point $e(z_{i_k})$. Then we have
\[
W(\mu_{n,\varepsilon},e_\#Q)
\le
\sum_{k=1}^m \int_{C_k} \|u-e(z_{i_k})\|\,\mathrm{d} e_\#Q(u).
\]
For any $u\in C_k$, since $C_k\subseteq B(u_k,\varepsilon)$ and $e(z_{i_k})\in B(u_k,\varepsilon)$,
\[
\|u-e(z_{i_k})\|
\le
\|u-u_k\|+\|u_k-\boldsymbol{e}(z_{i_k})\| < 2\varepsilon.
\]
Therefore,
\[
W(\mu_{n,\varepsilon},e_\#Q)
\le
\sum_{k=1}^m \int_{C_k} 2\varepsilon\, \mathrm{d} Q(u)
= 2\varepsilon \sum_{k=1}^m Q(C_k) = 2\varepsilon.
\]
Since $\mu_{n,\varepsilon}\in \mathcal M_n$, it follows that
\[
\inf_{\mu\in\mathcal M_n} W(\mu,e_\#Q)\le 2\varepsilon
\]
for all $n\ge N_\varepsilon$. Therefore, we derive
\[
\inf_{\mu\in\mathcal M_n} W(\mu,e_\#Q)\xrightarrow{a.s.}0.
\]
\end{proof}

\subsection{Proof of Theorem~\ref{the:unbias}}
\label{app:proof_unbias}

\begin{proof}
$r(z) = g^\ast(z) - \widehat{g}(z)$. We denote the reweighted distribution induced by OT (Eq.~\eqref{eq:OT_weights}) as
\[
e_\#P_{n,\boldsymbol w^\ast} = \sum_{i=1}^n w_i \delta_{e(z_i)}.
\]

We first show that $W(e_\#P_{n,\boldsymbol w^\ast}, e_\#Q) \to 0$. By the triangle inequality of the Wasserstein distance,
\begin{equation}
\label{equ1}
W(e_\#P_{n,\boldsymbol w^\ast}, e_\#Q) \le W(e_\#P_{n,\boldsymbol w^\ast}, e_\# Q_N) + W(e_\#Q_N, e_\#Q).
\end{equation}
Since $e_\#P_{n,\boldsymbol w^\ast}$ is defined in terms of the optimal solution
\[
\boldsymbol{w}^\ast \in \arg\min_{\mu\in\mathcal M_n} W(\mu,e_\# Q_N),
\]
we have
\begin{equation}
\label{equ2}
W(e_\#P_{n,\boldsymbol w^\ast}, e_\# Q_N) = \inf_{\mu\in\mathcal M_n} W(\mu,e_\# Q_N).   
\end{equation}

For any $\mu \in \mathcal M_n$, applying the triangle inequality again yields
\begin{equation*}
W(\mu,e_\# Q_N) \le W(\mu,e_\#Q) + W(e_\#Q,e_\# Q_N).    
\end{equation*}
Taking the infimum over $\mu \in \mathcal M_n$, we obtain
\begin{equation}
\label{equ3}
\inf_{\mu\in\mathcal M_n} W(\mu,e_\# Q_N) \le \inf_{\mu\in\mathcal M_n} W(\mu,e_\#Q) + W(e_\#Q,e_\# Q_N).
\end{equation}

Combining~\eqref{equ1},~\eqref{equ2} and~\eqref{equ3}, we obtain
\[
W(e_\#P_{n,\boldsymbol w^\ast}, e_\#Q)
\le
\inf_{\mu\in\mathcal M_n} W(\mu,e_\#Q) + 2\,W(e_\# Q_N,e_\#Q).
\]

Under Assumption~\ref{ass:bound_emb} that $\operatorname{supp}(e_\#Q)$ is compact, which must be bounded, we invoke Theorem 1 of \citet{fournier2015rate}, for $p=1$ and $d \ge 3$,
\begin{equation*}
\mathbb{E}[W(e_\#Q_N,e_\#Q)] \le C N^{-1/d},
\end{equation*}
where $C$ is a constant related to $p$. We now apply Markov's inequality. For any $M>0$,
\begin{equation*}
\mathbb{P}\!\left(W(e_\#Q_N,e_\#Q)>M N^{-1/d}\right) \le \frac{\mathbb{E}[W(e_\#Q_N,e_\#Q)]}{M N^{-1/d}}
\le \frac{C N^{-1/d}}{M N^{-1/d}} = \frac{C}{M}.
\end{equation*}
Hence, for any $\varepsilon>0$, choosing $M=C/\varepsilon$ yields
\begin{equation*}
\mathbb{P}\!\left(W(e_\#Q_N,e_\#Q)>M N^{-1/d}\right)\le \varepsilon,
\end{equation*}
which is exactly the definition of
\begin{equation*}
W(e_\#Q_N,e_\#Q)=O_p(N^{-1/d}).
\end{equation*}
Thus we have
\[
W(e_\#Q_N,e_\#Q)\xrightarrow{p}0.
\] By Lemma~\ref{lemma}, $\inf_{\mu\in\mathcal M_n} W(\mu,e_\#Q)\xrightarrow{a.s.}0$. Hence,
\[
W(e_\#P_{n,\boldsymbol w^\ast}, e_\#Q) \xrightarrow{p} 0.
\]

We now return to OTROPE:
\[
\widehat{V}_{\mathrm{OTROPE}}(\pi) =
\frac{1}{N}\sum_{j=1}^{N} \widehat{g}(\widetilde{z}_j)
+ \sum_{i=1}^n w_i^\ast r(z_i),
\]
which can be rewritten as
\[
\widehat{V}_{\mathrm{OTROPE}}(\pi) = \frac{1}{N}\sum_{j=1}^{N} \widehat{g}(z) +
\int r(z)\, \mathrm{d}P_{n,\boldsymbol w^\ast}(z) = \frac{1}{N}\sum_{j=1}^{N} \widehat{g}(\widetilde{z}_j) +
\int \widetilde{r}(e(z))\, \mathrm{d} e_\#P_{n,\boldsymbol w^\ast}(e(z)).
\]
Since $W(e_\#P_{n,\boldsymbol w^\ast}, e_\#Q) \xrightarrow{p} 0$, and $\widetilde{r}$ is $L_r$-Lipschitz continuous in the embedding space under Assumption~\ref{ass:lr}, the Kantorovich--Rubinstein duality implies that
\[
\left| \int \widetilde{r}(e(z))\, \mathrm{d}e_\#P_{n,\boldsymbol w^\ast}(e(z)) - \int \widetilde{r}(e(z))\, \mathrm{d} e_\#Q(e(z)) \right|
\le L_r\, W(e_\#P_{n,\boldsymbol w^\ast}, e_\#Q) \xrightarrow{p} 0.
\]
Moreover, for the first term of OTROPE, according to Weak Law of Large Numbers, we have
\[
\frac{1}{N}\sum_{j=1}^{N} \widehat{g}(\widetilde{z}_j) \xrightarrow{p} \mathbb{E}_{z\sim Q}[\widehat{g}(z)].
\]
Therefore,
\[
\widehat{V}_{\mathrm{OTROPE}}(\pi)
\xrightarrow{p} 
\mathbb{E}_{z\sim Q}[\widehat{g}(z)] + \mathbb{E}_{z\sim Q}[r(z)] = \mathbb{E}_{z\sim Q}[g^\ast(z)] = V^\ast.
\]
\end{proof}

\subsection{Proof of Theorem~\ref{the:ot_rate}}
\label{app:proof_rate}

\begin{proof}

As we just proofed in Appendix~\ref{app:proof_unbias}, 
\[
W(e_\#P_{n,\boldsymbol w^\ast},e_\#Q)
\le
\inf_{\mu\in\mathcal M_n} W(\mu,e_\#Q)
+
2W(e_\#Q_N,e_\#Q).
\]
We now control the first term with the additional assumption. For any $u\sim e_\#Q$, define its nearest neighbor among the source embeddings by
\[
R_n(u)\coloneq \min_{1\le i\le n}\|u-e(z_i)\|.
\]
Using the nearest-neighbor projection of $e_\#Q$ onto $\{e(z_i)\}_{i=1}^n$, there exists some $\mu_n\in\mathcal M_n$ such that
\[
W(\mu_n,e_\#Q)
\le
\mathbb E_{u\sim e_\#Q}[R_n(u)\vert z_1,\dots,z_n].
\]
This follows by constructing a transport plan that maps each point $u\sim e_\#Q$ to its nearest neighbor among $\{e(z_i)\}_{i=1}^n$, so that the transport cost is exactly the expected nearest-neighbor distance.
Hence,
\[
\inf_{\mu\in\mathcal M_n} W(\mu,e_\#Q)
\le
\mathbb E_{u\sim e_\#Q}[R_n(u)\vert z_1,\dots,z_n].
\]
Assumption~\ref{ass:ball} implies
\begin{equation}
\label{equ4}
\mathbb P(R_n(u)>\eta) = \left(1-e_\#P(B(u,\eta))\right)^n \le
(1-c\eta^{d_0})^n \le \exp(-cn\eta^{d_0})
\end{equation}
for all sufficiently small $r$. Since $\operatorname{supp}(e_\#Q)$ is compact, it has finite diameter. Hence, there exists a constant $D<\infty$ such that $R_n(u) \le D$ almost surely.
Using the tail-integral representation for nonnegative random variables, we have
\[
\mathbb E[R_n(u)] =
\int_0^\infty \mathbb P(R_n(u)>\eta)\, =
\int_0^D \mathbb P(R_n(u)>\eta)\mathrm{d}\eta.
\]
By~\eqref{equ4},
\[
\mathbb E[R_n(u)] \le \int_0^D \exp(-c n \eta^{d_0})\mathrm{d}\eta \le
\int_0^\infty \exp(-c n \eta^{d_0})\mathrm{d}\eta \le
\frac{1}{d_0}c^{-1/d_0} \Gamma\!\left(\frac{1}{d_0}\right) n^{-1/d_0}.
\]
Therefore,
\[
\mathbb E[R_n(u)] = O(n^{-1/d_0}).
\]
By Markov's inequality,
\[
\inf_{\mu\in\mathcal M_n} W(\mu,e_\#Q)
=
O_p(n^{-1/d_0}).
\]
We also have shown $W(e_\#Q_N,e_\#Q)=O_p(N^{-1/d})$ in Appendix~\ref{app:proof_unbias}.
Combining the two bounds gives
\[
W(e_\#P_{n,\boldsymbol w^\ast},e_\#Q) = O_p(n^{-1/d_0}+N^{-1/d}).
\]
We now decompose OTROPE:
\begin{equation}
\label{equ5}
\widehat{V}_{\mathrm{OTROPE}}(\pi)-V^\ast(\pi) =
\left[\frac1N\sum_{j=1}^N \widehat g(\widetilde z_j) -
\mathbb E_{z\sim Q}[\widehat g(z)]\right] +
\left[\sum_{i=1}^n w_i^\ast r(z_i) - \mathbb E_{z\sim Q}[r(z)]\right].    
\end{equation}
Since $\text{Var}_{\widetilde{z}\sim Q}(\widehat{g}(\widetilde{z}))$ is bounded, by Chebyshev's inequality, the first term 
\[
\frac1N\sum_{j=1}^N \widehat g(\widetilde z_j) -
\mathbb E_{z\sim Q}[\widehat g(z)] = O_p(N^{-1/2}).
\]
For the second term, 
\[
\left| \int r(z)\,\mathrm{d}P_{n,\boldsymbol w^\ast}(z) -
\int r(z)\,\mathrm{d}Q(z) \right| \le
L_r W(e_\#P_{n,\boldsymbol w^\ast},e_\#Q).
\]
Thus,
\[
\sum_{i=1}^n w_i^\ast r(z_i) - \mathbb E_{z\sim Q}[r(z)] = O_p(n^{-1/d_0}+N^{-1/d}).
\]
Combining the two terms yields
\[
\widehat{V}_{\mathrm{OTROPE}}(\pi)-V^\ast(\pi)
= O_p\!\left( N^{-1/2} + n^{-1/d_0} + N^{-1/d} \right).
\]

Additionally, suppose
\[
\|\widehat g_n-g^\ast\|_\infty=O_p(\alpha_n),
\qquad \alpha_n\to 0.
\]
Then $\|r\|_\infty=O_p(\alpha_n)$. By the decomposition~\eqref{equ5},
the first term is $O_p(N^{-1/2})$ by Chebyshev's inequality under bounded variance. For the second term, since $w_i^\ast\ge 0$ and $\sum_i w_i^\ast=1$,
\[
\left|
\sum_{i=1}^n w_i^\ast r(z_i)
\right|
\le
\sum_{i=1}^n w_i^\ast |r(z_i)|
\le
\|r\|_\infty
=
O_p(\alpha_n),
\]
and
\[
\left|
\mathbb E_{z\sim Q}[r(z)]
\right|
\le
\mathbb E_{z\sim Q}[|r(z)|]
\le
\|r\|_\infty
=
O_p(\alpha_n).
\]
Hence,
\[
\sum_{i=1}^n w_i^\ast r(z_i) - \mathbb E_{z\sim Q}[r(z)]
=
O_p(\alpha_n).
\]
Combining the two bounds gives
\[
\widehat{V}_{\mathrm{OTROPE}}(\pi)-V^\ast(\pi)
=
O_p(N^{-1/2}+\alpha_n).
\]
\end{proof}

\section{Training and Implementation Details}
\label{app:imple}

\subsection{Optimal Transport Weight Training}
\label{app:ot_optimize}
For ease of notation, let $E_1 = \{e_i\}_{i=1}^n = \{e(z_i)\}_{i=1}^n$, and $E_2 = \{\widetilde e_j\}_{j=1}^N = \{e(\widetilde z_j)\}_{j=1}^N$
denote the embedding representations of \(\mathcal{D}_1\) and \(\mathcal{D}_2\), respectively. If PCA is enabled, both \(E_1\) and \(E_2\) are transformed into a lower-dimensional OT space before computing the weights. 

We estimate the OT weights by solving an entropy-regularized optimal transport problem~\citep{cuturi2013sinkhorn} from the weighted empirical distribution on \(D_1\) to the uniform empirical distribution on \(D_2\)~ in \eqref{eq:OT_weights}.
The ground cost is the normalized pairwise Euclidean distance between embeddings:
\[
C_{ij} = \frac{\|e_i-\tilde e_j\|_2}
{\max_{i,j}\|e_i-\tilde e_j\|_2}.
\]
In implementation, the source weights are parameterized by unconstrained logits \(\theta \in \mathbb{R}^n\):
\[
w_i = \frac{\exp(\theta_i)} {\sum_{k=1}^n \exp(\theta_k)}.
\]
This softmax parameterization ensures \(w_i \ge 0\) and \(\sum_{i=1}^n w_i=1\) throughout optimization.

For a fixed \(\boldsymbol{w}\), the entropy-regularized OT cost is computed using Sinkhorn iterations. Specifically, the target marginal is fixed as
\[
b_j = \frac{1}{N},
\qquad j=1,\ldots,N,
\]
and the source marginal is \(\boldsymbol{w}\). Define
\[
M_{ij} = -\frac{C_{ij}}{\lambda},
\]
where $\lambda$ controls the strength of entropy regularization. 
Smaller values of $\lambda$ lead to less regularized and more concentrated transport plans, whereas larger values yield smoother transport plans.
The Sinkhorn algorithm alternates between updating two log-domain scaling variables \(f \in \mathbb{R}^n\) and \(g \in \mathbb{R}^N\):
\[
f_i = \log w_i - \log \sum_{j=1}^N \exp(M_{ij}+h_j),
\]
and
\[
h_j = \log b_j - \log \sum_{i=1}^n \exp(M_{ij}+f_i).
\]
These updates are repeated for a fixed number of inner Sinkhorn iterations, or until the change in \(f\) is below a numerical tolerance.

After the Sinkhorn iterations converge, the transport plan is recovered as
\[
\Gamma_{ij} = \exp(f_i + h_j + M_{ij}).
\]
The differentiable OT objective is then computed as
\[
\langle \Gamma, C\rangle
=
\sum_{i=1}^n \sum_{j=1}^N \Gamma_{ij} C_{ij}.
\]
The logits \(\theta\) are initialized at zero, corresponding to uniform source weights. We then minimize the Sinkhorn OT cost using the Adam optimizer for a fixed number of outer iterations. In each outer iteration, the current logits determine \(\boldsymbol{w}\), the inner Sinkhorn iterations compute the transport plan \(\Gamma\), and the resulting OT cost is backpropagated to update \(\theta\). The final OT weights \(\boldsymbol{w}^\ast\) are taken from the iterate with the smallest observed OT cost.

\subsection{MoE Architecture and Optimization}\label{sec:moe}
We train the MoE model using a two-layer multilayer perceptron (MLP). For each input sample, the MLP takes the corresponding embedding vector as input and outputs a set of expert-combination weights. These weights are then used to aggregate the predictions from multiple fixed experts into a single MoE prediction.

The MoE model is trained on the labeled source data. The expert predictions are treated as fixed inputs, while the MLP learns the aggregation rule by matching the final MoE prediction to the observed label. For binary outcomes, we optimize the MLP parameters by minimizing the binary cross-entropy loss:
\[
\mathcal{L}_{\mathrm{MoE}}
= -\frac{1}{n}\sum_{i=1}^n
\left[
g^\ast(z_i) \log \widehat{g}(z_i)
+
\{1-g^\ast(z_i)\}\log\{1-\widehat{g}(z_i)\}
\right],
\]
where $\widehat{g}(z_i)$ denotes the MoE prediction for sample $z_i$, and $g^\ast(z_i)$ denotes the corresponding ground-truth binary label.

After training, we fix the MLP-based MoE model and use it to generate predictions on both the labeled source data and the target data. These predictions are saved and used in the downstream estimation step. Specifically, the target MoE predictions provide the plug-in estimate, while the source MoE predictions are used to construct the residual correction.

\section{Experimental Details}
\label{app:exp}

All experiments are conducted on a server equipped with an NVIDIA RTX A6000 GPU. 
Each repetition takes approximately 30 seconds to 5 minutes, depending on the sample size and target-policy setting.


\subsection{Baseline Estimators}
\label{app:baseline}

\paragraph{IS Estimator.}
The IS estimator~\citep{precup2000eligibility} is defined as
\begin{equation}
\label{equ:IS}
\widehat{V}_{\text{IS}}(\pi)= \frac{1}{n} \sum_{i=1}^n 
   \frac{\pi(y_i \vert x_i)}{\pi_0(y_i \vert x_i)} \, g^\ast(x_i, y_i).
\end{equation}

\paragraph{PPI++ Estimator.}
The PPI++ estimator~\citep{angelopoulos2023ppi++} is defined as
\begin{equation}
\label{equ:ppi++}
\widehat{V}_{\text{PPI++}}(\pi)
=\frac{1}{n} \sum_{i=1}^n g^\ast(z_i) 
+ \widehat{\lambda}\left(\frac{1}{N}\sum_{j=1}^{N}\widehat{g}(\widetilde{z}_j)
- \frac{1}{n} \sum_{i=1}^n \widehat{g}(z_i)\right),
\end{equation}
where
\[
\widehat{\lambda}
=\frac{\widehat{\operatorname{Cov}}_n\left(g^\ast(z_i),\widehat{g}(z_i)\right)}
{(1+\frac{n}{N})\widehat{\operatorname{Var}}_{n+N}\left(\widehat{g}(z_i),\widehat{g}(\widetilde{z}_j)\right)}.
\]
Here, $\widehat{\operatorname{Cov}}_n$ denotes the empirical covariance computed over the labeled samples $\{z_i\}_{i=1}^n$, and $\widehat{\operatorname{Var}}_{n+N}$ denotes the empirical variance of $\widehat{g}$ computed over the combined sample $\{z_i\}_{i=1}^n \cup \{\widetilde{z}_j\}_{j=1}^N$.

\paragraph{CLS Estimator for Density Ratio.}
When the densities $\pi(y \vert x)$ and $\pi_0(y \vert x)$ are unavailable, we estimate the importance weights using a probabilistic classifier. To this end, we construct a binary classification problem in which samples are drawn from a mixture of two conditional distributions: the target policy $\pi(\cdot \vert x)$ and the behavior policy $\pi_0(\cdot \vert x)$. Specifically, we assign label $l=1$ if $y \sim \pi(\cdot \vert x)$ and $l=0$ if $y \sim \pi_0(\cdot \vert x)$, with class priors $P(l=1)=\pi_1$ and $P(l=0)=\pi_0$. 
Following~\citep{sugiyama2012density}, the density ratio can be expressed as
\[
\frac{\pi(y \vert x)}{\pi_0(y \vert x)}
= \frac{\pi_0}{\pi_1} \cdot \frac{P(l=1 \vert y,x)}{P(l=0 \vert y,x)}.
\]
In our experiment, we apply 
\[
\frac{\pi_0}{\pi_1} = \frac{n}{N}.
\]
We use this plug-in CLS density ratio estimator in IS and DR method in Section~\ref{subsec:real}.

\paragraph{OT Estimator.}
We consider a direct OT-based estimator,
\begin{equation}
\label{equ:ot}
\widehat{V}_{\text{OT}}(\pi)= \frac{1}{n} \sum_{i=1}^n w_i^\ast \, g^\ast(x_i, y_i),
\end{equation}
where $w_i^\ast$ is the same as in OTROPE.

\subsection{Data Construction}\label{asec:data_cons}

\subsubsection{\textsc{RewardBench~2}}

\begin{table}[h]
\centering
\caption{Two policies evaluated in \textsc{RewardBench~2}.}
\label{tab:rb-comparison}
\resizebox{\textwidth}{!}{%
\begin{tabular}{lll}
\toprule
\textbf{Policy Name} & \textbf{Claude 3.5 Sonnet} & \textbf{Mixed} \\
\midrule
\textbf{Target Policy} 
  & \texttt{claude-3-5-sonnet-20241022} 
  & \texttt{claude-3-5-sonnet-20241022}, \texttt{human} \\
\midrule
\textbf{Behavior Policy} 
  & \texttt{Llama-3.1-70B-Instruct}, \texttt{Qwen2.5-72B-Instruct},
  & \texttt{Llama-3.1-70B-Instruct}, \texttt{Qwen2.5-72B-Instruct}, \\
  & \texttt{Qwen2.5-7B-Instruct}
  & \texttt{Qwen2.5-7B-Instruct} \\
\midrule
\textbf{Opponent Pool}
  & \texttt{Llama-3.1-70B-Instruct}, \texttt{Llama-3.1-8B-Instruct},
  & \texttt{Llama-3.1-70B-Instruct}, \texttt{Llama-3.1-8B-Instruct}, \\
  & \texttt{Llama-3.1-Tulu-3-8B}, \texttt{Mistral-7B-Instruct-v0.3},
  & \texttt{Llama-3.1-Tulu-3-8B}, \texttt{Mistral-7B-Instruct-v0.3}, \\
  & \texttt{Qwen2.5-72B-Instruct}, \texttt{Qwen2.5-7B-Instruct},
  & \texttt{Qwen2.5-72B-Instruct}, \texttt{Qwen2.5-7B-Instruct}, \\
  & \texttt{gpt-4o-2024-08-06}, \texttt{human}
  & \textbf{\texttt{claude-3-5-sonnet-20241022}}, \texttt{gpt-4o-2024-08-06}, \texttt{human} \\
\midrule
\textbf{Size $\mathcal{D}_1$} & 1{,}408 & 1{,}723 \\
\textbf{Size $\mathcal{D}_2$}  & 493     & 591     \\
\textbf{Size (total)} & 1{,}901 & 2{,}314 \\
\bottomrule
\end{tabular}%
}
\end{table}



We randomly sample $n$ pairs from $P$ to form $\mathcal{D}_1$, and use all pairs from $Q$ as $\mathcal{D}_2$. 
We assume that the ground-truth preference $g^\ast$ is observed for samples in $\mathcal{D}_1$ but unavailable for those in $\mathcal{D}_2$. 
The target value $V^\ast(\pi)$ is computed as the average ground-truth preference over $\mathcal{D}_2$. 
The true values are $0.843$ for Claude 3.5 Sonnet and $0.870$ for the Mixed policy. 
In each repetition, we set $n=1400$ for Claude 3.5 Sonnet and $n=1600$ for the Mixed policy in Section~\ref{subsec:real}.
Because the number of samples for each comparison pair in the original dataset is limited, we allow the target and behavior policies to also appear in the opponent pools, thereby constructing a larger opponent policy pool.

\subsubsection{\textsc{Chatbot Arena}}

\begin{table}[h]
\centering
\caption{Two policies evaluated in \textsc{Chatbot Arena}.}
\label{tab:chatbot-comparison}
\resizebox{\textwidth}{!}{%
\begin{tabular}{lll}
\toprule
\textbf{Policy Name} & \textbf{GPT-4} & \textbf{GPT-3.5-Turbo} \\
\midrule
\textbf{Target Policy}
  & \texttt{gpt-4}
  & \texttt{gpt-3.5-turbo} \\
\midrule
\textbf{Behavior Policy}
  & \texttt{koala-13b}
  & \texttt{koala-13b} \\
\midrule
\textbf{Opponent Pool}
  & \texttt{RWKV-4-Raven-14B}, \texttt{alpaca-13b},
  & \texttt{RWKV-4-Raven-14B}, \texttt{alpaca-13b}, \\
  & \texttt{stablelm-tuned-alpha-7b}, \texttt{vicuna-13b}
  & \texttt{stablelm-tuned-alpha-7b}, \texttt{vicuna-13b} \\
\midrule
\textbf{Size $\mathcal{D}_1$} & 1{,}615 & 1{,}615 \\
\textbf{Size $\mathcal{D}_2$}  & 998     & 1{,}080 \\
\textbf{Size (total)} & 2{,}613 & 2{,}695 \\
\bottomrule
\end{tabular}%
}
\end{table}

The ground-truth values of the target policy are $0.755$ for GPT-4 and $0.607$ for GPT-3.5-Turbo. 
We randomly sample $n=1600$ examples in Section~\ref{subsec:real}.

\subsection{LLM-as-a-Judge}\label{asec:llms}

We employ five LLMs to provide evaluation judgments: Phi-4-mini-instruct, gemma-7b-it, deepseek-llm-7b-chat, Llama-3.1-8B-Instruct, and Qwen2.5-14B-Instruct. These models are drawn from different families and may exhibit complementary strengths in evaluating diverse tasks. We also include a strong LLM, DeepSeek-V3.1, as a comparison.



\subsection{Hyperpameters}\label{asec:hyper}

Since \textsc{RewardBench~2} involves behavior and target policies induced by pools of LLMs, rather than by a single LLM as in \textsc{Chatbot Arena}, its induced data distribution is more complex. To improve optimization stability, we therefore use more training iterations and a smaller learning rate for \textsc{RewardBench~2}.

\begin{table}[h]
\centering
\caption{Hyperparameters for MoE training on \textsc{RewardBench~2} and \textsc{Chatbot Arena}.}
\label{tab:moe-hyperparameters}
\begin{tabular}{lcc}
\toprule
Hyperparameter & \textsc{RewardBench~2} & \textsc{Chatbot Arena} \\
\midrule
Hidden dimensions & $(64, 32)$ & $(64, 32)$ \\
Optimizer & Adam & Adam \\
Learning rate & $3\times 10^{-4}$ & $1\times 10^{-3}$ \\
Batch size & $256$ & $256$ \\
Number of epochs & $50$ & $50$ \\
Temperature & $\{0.3, 0.5, 0.7, 1.0\}$ & $\{0.3, 0.5, 0.7, 1.0\}$ \\
Regularization & $\{0, 0.001, 0.01, 0.05, 0.1\}$ & $\{0, 0.001, 0.01, 0.05, 0.1\}$ \\
Training/validation split & $80\%/20\%$ & $80\%/20\%$ \\
\bottomrule
\end{tabular}
\end{table}
We select the MoE hyperparameters via grid search over the regularization and temperature parameters (see Table~\ref{tab:moe-hyperparameters}).

\begin{table}[h]
\centering
\caption{Hyperparameters used for OT weight optimization in the controlled experiments, under both fixed and trained $\widehat{g}$ settings.}
\label{tab:ot-simu}
\begin{tabular}{lc}
\toprule
Hyperparameter & Value / Description \\
\midrule
Entropic regularization & $0.1$ \\
Number of Sinkhorn iterations & $1000$ \\
Outer optimizer & Adam \\
Number of outer iterations & $500$ \\
Learning rate & $0.05$ \\
\bottomrule
\end{tabular}
\end{table}

\begin{table}[h]
\centering
\caption{Hyperparameters for OT weight optimization on \textsc{RewardBench~2} and \textsc{Chatbot Arena}.}
\label{tab:ot-real-hparams}
\begin{tabular}{lc}
\toprule
Hyperparameter & Values \\
\midrule
PCA dimension & $128$ \\
Number of Sinkhorn iterations & $2000$\\
Outer optimizer & Adam \\
Number of outer iterations & $1000$\\
Learning rate & $0.01$ \\
\bottomrule
\end{tabular}
\end{table}

We set the entropic regularization parameter to $\lambda=0.1$ for Claude 3.5 Sonnet and $\lambda=0.08$ for the Mixed policy on \textsc{RewardBench~2}. 
For \textsc{Chatbot Arena}, we use $\lambda=0.05$ for GPT-4 and $\lambda=0.1$ for GPT-3.5-Turbo.
We choose $\lambda$ within the empirically stable range $[0.05,0.10]$ according to the matching difficulty of each target-policy setting and the computation efficiency. 
For more heterogeneous or distributionally distinct target policies, such as the Mixed policy and GPT-4, we use a smaller $\lambda$ to preserve finer pairwise cost information and improve matching resolution. 
For smoother single-model settings, such as GPT-3.5-Turbo, we use the default value $\lambda=0.10$, which provides better numerical smoothness and computational efficiency.

\section{Additional Experimental Results}
\label{app:more_results}

\subsection{Toy Example}
\label{app:toy}

We first provide additional numerical results in Table~\ref{tab:toy_results} for the toy example in Section~\ref{subsec:toy}.

\newpage 

\begin{table*}[h]
\centering
\setlength{\tabcolsep}{4pt}
\caption{Mean absolute error (MAE) over 200 replications under varying distribution shift $\tau$. Values are reported as mean with standard deviation in parentheses.}
\label{tab:toy_results}
\begin{tabular}{c ccccccc}
\toprule
$\tau$ & IS & DR & PPI & PPI++ & OTROPE \\
\midrule
1 & \ms{0.908}{1.077} & \ms{0.959}{1.186} & \ms{2.026}{0.126} & \ms{2.013}{0.063} & \ms{\best{0.236}}{0.187} \\
2 & \ms{6.906}{18.068} & \ms{6.420}{16.300} & \ms{4.005}{0.131} & \ms{3.998}{0.069} & \ms{\best{1.037}}{0.703} \\
3 & \ms{8.680}{8.832} & \ms{8.232}{5.112} & \ms{5.992}{0.138} & \ms{5.989}{0.070} & \ms{\best{2.880}}{1.229} \\
4 & \ms{9.865}{0.227} & \ms{9.878}{0.023} & \ms{7.886}{0.138} & \ms{7.894}{0.072} & \ms{\best{4.635}}{1.740} \\
5 & \ms{11.450}{0.000} & \ms{11.450}{0.022} & \ms{9.459}{0.118} & \ms{9.452}{0.073} & \ms{\best{6.205}}{1.969} \\
\bottomrule
\end{tabular}
\end{table*}

\subsection{More Results for Controlled Experiments}
\label{app:simu}

Here we show more results for the controlled experiments in Figure~\ref{fig:simu} of Section~\ref{subsec:simu}.

\begin{table*}[h]
\caption{Mean absolute error and standard deviation of target policy value estimation with a fixed predictor. The labeled sample size $n$ varies.}
\label{tab:simu_fixed}
\centering
\scriptsize
\setlength{\tabcolsep}{2pt}
\begin{tabular}{l ccccc}
\toprule
\multirow{2}{*}[-0.5ex]{\textbf{Estimator}} &
\multicolumn{5}{c}{$n$} \\
\cmidrule{2-6}
& 200 & 400 & 600 & 800 & 1000 \\
\midrule

\bluecell \multicolumn{6}{c}{$\tau=1$} \\
\midrule

\greencell IS
& 1.944 {\scriptsize(5.290)} & 1.463 {\scriptsize(2.548)}
& 1.183 {\scriptsize(1.724)} & 1.085 {\scriptsize(1.395)}
& 0.924 {\scriptsize(1.208)} \\

\greencell DR
& 2.464 {\scriptsize(6.989)} & 1.776 {\scriptsize(3.272)}
& 1.446 {\scriptsize(2.233)} & 1.300 {\scriptsize(1.890)}
& 1.139 {\scriptsize(1.656)} \\

\yellowcell PPI
& 2.048 {\scriptsize(0.426)} & 2.079 {\scriptsize(0.316)}
& 2.055 {\scriptsize(0.270)} & 2.057 {\scriptsize(0.240)}
& 2.057 {\scriptsize(0.191)} \\

\yellowcell PPI++
& 2.009 {\scriptsize(0.120)} & 2.018 {\scriptsize(0.087)}
& 2.016 {\scriptsize(0.076)} & 2.021 {\scriptsize(0.068)}
& 2.018 {\scriptsize(0.061)} \\

\redcell \textbf{OTROPE}
& \best{0.299} {\scriptsize(0.257)} & \best{0.249} {\scriptsize(0.181)}
& \best{0.219} {\scriptsize(0.171)} & \best{0.198} {\scriptsize(0.164)}
& \best{0.195} {\scriptsize(0.163)} \\

\bluecell \multicolumn{6}{c}{$\tau=2$} \\
\midrule

\greencell IS
& 16.966 {\scriptsize(100.571)} & 11.194 {\scriptsize(49.746)}
& 8.919 {\scriptsize(32.854)} & 7.744 {\scriptsize(24.632)}
& 7.403 {\scriptsize(20.115)} \\

\greencell DR
& 18.267 {\scriptsize(112.970)} & 11.864 {\scriptsize(55.908)}
& 9.377 {\scriptsize(36.970)} & 8.204 {\scriptsize(27.769)}
& 7.847 {\scriptsize(22.789)} \\

\yellowcell PPI
& 4.024 {\scriptsize(0.288)} & 4.047 {\scriptsize(0.210)}
& 4.030 {\scriptsize(0.178)} & 4.032 {\scriptsize(0.157)}
& 4.030 {\scriptsize(0.122)} \\

\yellowcell PPI++
& 3.999 {\scriptsize(0.120)} & 4.007 {\scriptsize(0.087)}
& 4.006 {\scriptsize(0.076)} & 4.010 {\scriptsize(0.068)}
& 4.008 {\scriptsize(0.061)} \\

\redcell \textbf{OTROPE}
& \best{1.529} {\scriptsize(0.711)} & \best{1.250} {\scriptsize(0.684)}
& \best{1.181} {\scriptsize(0.673)} & \best{0.993} {\scriptsize(0.578)}
& \best{0.959} {\scriptsize(0.613)} \\

\midrule
\bluecell \multicolumn{6}{c}{$\tau=3$} \\
\midrule

\greencell IS
& 11.720 {\scriptsize(37.407)} & 9.535 {\scriptsize(17.959)}
& 8.895 {\scriptsize(11.463)} & 8.549 {\scriptsize(8.226)}
& 8.324 {\scriptsize(6.286)} \\

\greencell DR
& 13.186 {\scriptsize(51.962)} & 10.286 {\scriptsize(25.216)}
& 9.378 {\scriptsize(16.302)} & 8.891 {\scriptsize(11.860)}
& 8.652 {\scriptsize(9.186)} \\

\yellowcell PPI
& 6.033 {\scriptsize(0.426)} & 6.064 {\scriptsize(0.316)}
& 6.040 {\scriptsize(0.270)} & 6.042 {\scriptsize(0.240)}
& 6.042 {\scriptsize(0.191)} \\

\yellowcell PPI++
& 5.994 {\scriptsize(0.120)} & 6.003 {\scriptsize(0.087)}
& 6.001 {\scriptsize(0.076)} & 6.006 {\scriptsize(0.068)}
& 6.003 {\scriptsize(0.061)} \\

\redcell \textbf{OTROPE}
& \best{3.601} {\scriptsize(1.549)} & \best{3.166} {\scriptsize(1.489)}
& \best{3.119} {\scriptsize(1.440)} & \best{2.860} {\scriptsize(1.373)}
& \best{2.711} {\scriptsize(1.403)} \\

\bottomrule
\end{tabular}
\end{table*}

We also report results for a predictor $\widehat g(x,y)$ trained on $\mathcal{D}_1$ in Table~\ref{tab:simu_trained} and Figure~\ref{fig:simu_trained}. Specifically, $\widehat g(x,y)$ is implemented as a linear model trained on $\mathcal{D}_1$ using only the covariate $x$.

\begin{figure}[h]
    \centering
    \includegraphics[width=0.8\linewidth]{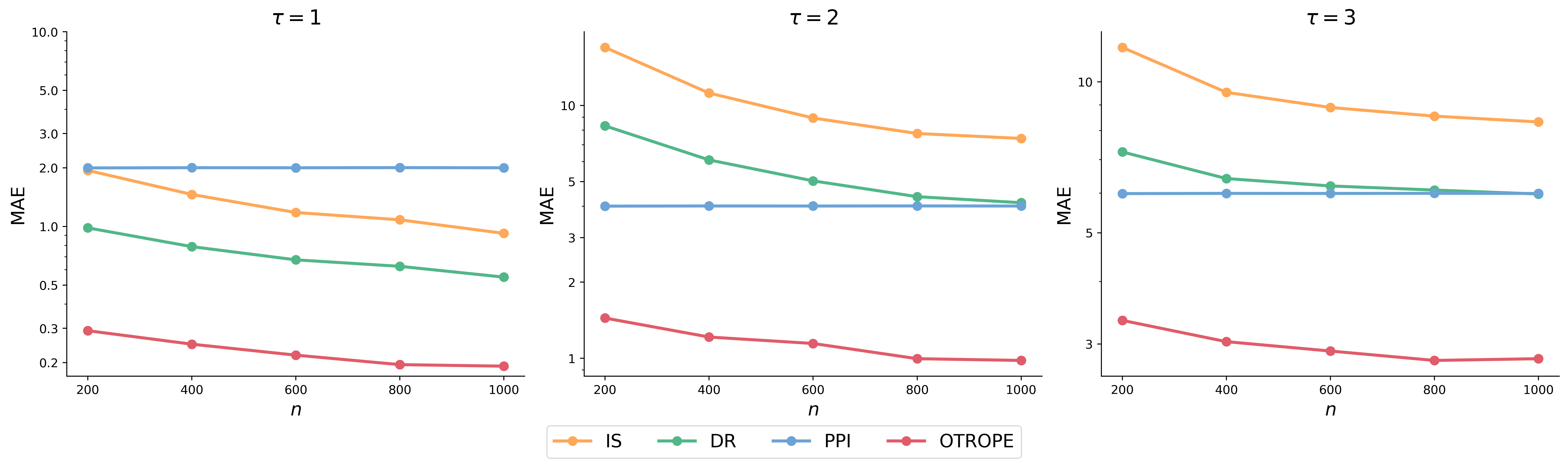}
    \caption{Log-scale mean absolute error under the predictor $\widehat g$ trained on $\mathcal{D}_1$ for distribution shift levels $\tau \in \{1,2,3\}$.
    PPI++ is omitted since it nearly overlaps with PPI.}
    \label{fig:simu_trained}
\end{figure}

\begin{table*}[h]
\caption{Mean absolute error and standard deviation of target policy value estimation with a predictor trained on $\mathcal{D}_1$. The labeled sample size $n$ varies.}
\label{tab:simu_trained}
\centering
\scriptsize
\setlength{\tabcolsep}{2pt}
\begin{tabular}{l ccccc}
\toprule
\multirow{2}{*}[-0.5ex]{\textbf{Estimator}} &
\multicolumn{5}{c}{$n$} \\
\cmidrule{2-6}
& 200 & 400 & 600 & 800 & 1000 \\
\midrule

\bluecell \multicolumn{6}{c}{$\tau=1$} \\
\midrule

\greencell IS
& 1.940 {\scriptsize(5.291)} & 1.460 {\scriptsize(2.550)}
& 1.179 {\scriptsize(1.726)} & 1.082 {\scriptsize(1.398)}
& 0.922 {\scriptsize(1.210)} \\

\greencell DR
& 0.984 {\scriptsize(2.234)} & 0.788 {\scriptsize(1.230)}
& 0.674 {\scriptsize(0.839)} & 0.624 {\scriptsize(0.648)}
& 0.549 {\scriptsize(0.524)} \\

\yellowcell PPI
& 1.998 {\scriptsize(0.121)} & 2.003 {\scriptsize(0.093)}
& 2.001 {\scriptsize(0.083)} & 2.005 {\scriptsize(0.078)}
& 2.001 {\scriptsize(0.075)} \\

\yellowcell PPI++
& 2.000 {\scriptsize(0.120)} & 2.009 {\scriptsize(0.087)}
& 2.007 {\scriptsize(0.076)} & 2.012 {\scriptsize(0.068)}
& 2.009 {\scriptsize(0.061)} \\

\redcell \textbf{OTROPE}
& \best{0.291} {\scriptsize(0.258)} & \best{0.249} {\scriptsize(0.177)}
& \best{0.218} {\scriptsize(0.167)} & \best{0.195} {\scriptsize(0.160)}
& \best{0.192} {\scriptsize(0.161)} \\

\bluecell \multicolumn{6}{c}{$\tau=2$} \\
\midrule

\greencell IS
& 16.967 {\scriptsize(100.571)} & 11.196 {\scriptsize(49.746)}
& 8.921 {\scriptsize(32.854)} & 7.745 {\scriptsize(24.631)}
& 7.405 {\scriptsize(20.115)} \\

\greencell DR
& 8.305 {\scriptsize(36.887)} & 6.089 {\scriptsize(18.829)}
& 5.031 {\scriptsize(12.417)} & 4.357 {\scriptsize(9.182)}
& 4.124 {\scriptsize(7.273)} \\

\yellowcell PPI
& 3.998 {\scriptsize(0.121)} & 4.003 {\scriptsize(0.093)}
& 4.002 {\scriptsize(0.083)} & 4.005 {\scriptsize(0.078)}
& 4.001 {\scriptsize(0.075)} \\

\yellowcell PPI++
& 4.000 {\scriptsize(0.120)} & 4.009 {\scriptsize(0.087)}
& 4.008 {\scriptsize(0.076)} & 4.012 {\scriptsize(0.068)}
& 4.010 {\scriptsize(0.061)} \\

\redcell \textbf{OTROPE}
& \best{1.443} {\scriptsize(0.701)} & \best{1.214} {\scriptsize(0.659)}
& \best{1.144} {\scriptsize(0.657)} & \best{0.996} {\scriptsize(0.580)}
& \best{0.981} {\scriptsize(0.573)} \\

\midrule
\bluecell \multicolumn{6}{c}{$\tau=3$} \\
\midrule

\greencell IS
& 11.714 {\scriptsize(37.408)} & 9.529 {\scriptsize(17.960)}
& 8.888 {\scriptsize(11.465)} & 8.542 {\scriptsize(8.228)}
& 8.318 {\scriptsize(6.287)} \\

\greencell DR
& 7.247 {\scriptsize(12.901)} & 6.413 {\scriptsize(6.184)}
& 6.199 {\scriptsize(3.775)} & 6.083 {\scriptsize(2.541)}
& 5.973 {\scriptsize(1.858)} \\

\yellowcell PPI
& 5.985 {\scriptsize(0.121)} & 5.991 {\scriptsize(0.093)}
& 5.989 {\scriptsize(0.083)} & 5.992 {\scriptsize(0.078)}
& 5.988 {\scriptsize(0.075)} \\

\yellowcell PPI++
& 5.988 {\scriptsize(0.120)} & 5.996 {\scriptsize(0.087)}
& 5.995 {\scriptsize(0.076)} & 5.999 {\scriptsize(0.068)}
& 5.997 {\scriptsize(0.061)} \\

\redcell \textbf{OTROPE}
& \best{3.344} {\scriptsize(0.873)} & \best{3.032} {\scriptsize(0.885)}
& \best{2.904} {\scriptsize(0.856)} & \best{2.781} {\scriptsize(0.763)}
& \best{2.804} {\scriptsize(0.768)} \\

\bottomrule
\end{tabular}
\end{table*}

Together with Table~\ref{tab:simu_fixed}, these results show that OTROPE outperforms all baselines across the considered settings. This is consistent with Theorems~\ref{the:unbias} and~\ref{the:ot_rate}, which impose no specific structural requirement on the relationship between the learned OT weights and the predictor $\widehat g$ or the induced residual function. Thus, $\widehat g$ can be an arbitrary model, including one trained on $\mathcal D_1$.

\subsection{Numerical Verification of Proposition~\ref{prop:dr_bad}}
\label{app:prop_verification}

We provide a direct numerical verification of Proposition~\ref{prop:dr_bad} using simulated data with oracle density ratios. We consider $n=1000$ labeled samples and $N=2000$ unlabeled samples, with 200 replications for each $\tau \in \{1,\ldots,5\}$ and $d_n=0.05$. We set
\[
T_n = \frac{1}{n}\sum_{i=1}^n \rho(z_i)
\]
and choose $a_n$ as its empirical 95th percentile. We also set the direct-method bias $|\widehat V_{\mathrm{DM}}-V^\ast|$ to $b_n$ at its empirical 5th percentile. These choices match the two high-probability conditions assumed in Proposition~\ref{prop:dr_bad}.

\textbf{Bias bound.}
Among the replications where both conditions hold simultaneously, corresponding to the $1-2d_n$ guarantee, we compare the predicted lower bound $b_n-C_r a_n$ with the smallest observed value of $|\widehat V_{\mathrm{DR}}-V^\ast|$ among those replications. Table~\ref{tab:prop_bias_check} reports the results. The bound holds in all cases.

\begin{table}[ht]
\centering
\caption{Numerical verification of the bias bound in Proposition~\ref{prop:dr_bad}. The constant $C_r=16.51$ is a pooled upper bound on $|r(z)|$ across all $\tau$.}
\label{tab:prop_bias_check}
\setlength{\tabcolsep}{6pt}
\begin{tabular}{lccccc}
\toprule
$\tau$ & $a_n$ & $b_n$ & $b_n-C_r a_n$ & $\min |\widehat V_{\mathrm{DR}}-V^\ast|$ & Holds \\
\midrule
1 & 1.373  & 3.965  & $-18.70$ & 0.008  & 100\% \\
2 & 2.719  & 5.959  & $-38.94$ & 0.071  & 100\% \\
3 & 0.177  & 7.948  & 5.03    & 6.119  & 100\% \\
4 & 0.0006 & 9.847  & 9.838   & 9.847  & 100\% \\
5 & 0.0000 & 11.415 & 11.415  & 11.415 & 100\% \\
\bottomrule
\end{tabular}
\end{table}

\textbf{Variance bound.}
The variance lower bound in Proposition~\ref{prop:dr_bad} additionally requires $0<c_r\le r(z)$. The original fixed predictor $\widehat g(x,y)=-x$ does not satisfy this condition. Therefore, for this diagnostic, we use a uniformly conservative predictor $\widehat g(x,y)=-x-20$, which ensures $c_r>0$, and compare the empirical variance $\mathrm{Var}(\widehat V_{\mathrm{DR}})$ with the predicted lower bound in Proposition~\ref{prop:dr_bad}. The results are shown in Table~\ref{tab:prop_var_check}. The bound is confirmed in the regimes $\tau=1,2,3$.

\begin{table}[t]
\centering
\caption{Numerical verification of the variance lower bound in Proposition~\ref{prop:dr_bad} under a conservative predictor satisfying $0<c_r\le r(z)$.}
\label{tab:prop_var_check}
\setlength{\tabcolsep}{7pt}
\begin{tabular}{lcccc}
\toprule
$\tau$ & $\widehat\kappa$ & $\mathrm{Var}(\widehat V_{\mathrm{DR}})$ & Predicted lower bound & Holds \\
\midrule
1 & 0.000 & 30.5              & $-1.0\times 10^3$ & Yes \\
2 & 0.000 & $1.62\times 10^4$ & $5.57\times 10^3$ & Yes \\
3 & 0.000 & $7.97\times 10^3$ & 250               & Yes \\
\bottomrule
\end{tabular}
\end{table}

The bound is cleanly confirmed in the regime targeted by the proposition, particularly for $\tau=1,2,3$, where $R_n$ still varies meaningfully across replications. For $\tau=4,5$, however, $R_n$ is nearly identical across most of the 200 replications. As a result, the empirical quantiles and covariance terms used to estimate $a_n$, $b_n$, and $\widehat\kappa$ become unstable and no longer provide reliable empirical proxies for the population-level quantities appearing in Proposition~\ref{prop:dr_bad}.

\subsection{Robustness under Variance Shift}
\label{app:variance_shift}

Our main simulation varies the discrepancy between the behavior and target distributions primarily through a mean or location shift. This design is natural because the policy value is an expectation, and a location shift along an outcome-relevant direction provides a direct and interpretable way to control changes that affect the target policy value. However, this is only a choice of simulation design rather than a restriction of OTROPE. The OT objective aligns the full empirical distributions and does not assume that the behavior--target discrepancy arises only from the mean. Therefore, variance, scale, and more general distributional-shape shifts are also covered by the proposed framework.

To evaluate OTROPE under a variance or scale shift, we consider an additional simulation with $n=N=1000$ and $d=5$, where only the conditional scale changes while the conditional mean remains unchanged:
\[
P:\ y_k \mid x \sim \mathrm{TruncNormal}(\gamma_k x,\sigma_{\mathrm{ref}},[-4,4]),
\,\,
Q:\ y_k \mid x \sim \mathrm{TruncNormal}(\gamma_k x,\sigma_{\mathrm{tgt}},[-4,4]).
\]
We fix $\sigma_{\mathrm{ref}}=0.5$ and consider two scale ratios, $\sigma_{\mathrm{tgt}}/\sigma_{\mathrm{ref}}=0.5$ and $2.0$. All other settings follow the toy example in the main paper. We use 20 replications and report the mean absolute error of the estimated target policy value in Table~\ref{tab:variance_shift}; standard deviations are shown in parentheses.

\begin{table}[ht]
\centering
\caption{Mean absolute error under variance-shift simulations. Standard deviations are reported in parentheses.}
\label{tab:variance_shift}
\setlength{\tabcolsep}{6pt}
\begin{tabular}{lcccc}
\toprule
$\sigma_{\mathrm{tgt}}/\sigma_{\mathrm{ref}}$ & IS & DR & PPI & OTROPE \\
\midrule
0.50 & 0.226 (0.066) & 0.287 (0.101) & 0.083 (0.005) & 0.104 (0.007) \\
2.00 & 2.163 (30.56) & 2.138 (24.19) & 0.120 (0.007) & 0.125 (0.013) \\
\bottomrule
\end{tabular}
\end{table}

OTROPE remains substantially more accurate and stable than IS and DR under both scale-shift settings. PPI performs slightly better in this particular experiment because the conditional mean structure is unchanged, making the setting especially favorable to its unweighted residual correction. When $\sigma_{\mathrm{tgt}}/\sigma_{\mathrm{ref}}=2$, IS and DR become unstable because the target distribution has heavier tails than the behavior distribution; consequently, a small number of tail observations can receive very large importance weights and dominate the estimate. In contrast, PPI uses uniform weights, while OTROPE uses simplex-constrained OT weights rather than raw density ratios. Overall, the results indicate that OTROPE remains competitive and stable, and that its effectiveness is not limited to mean-shift settings.

\subsection{Empirical Support for Assumption~\ref{ass:bound_emb}}
\label{app:support}

Assumption~\ref{ass:bound_emb} is a population-level support condition and cannot be formally verified from finite samples. Nevertheless, its plausibility can be empirically assessed through geometric coverage diagnostics in the embedding space.
For each target-policy sample, we compare its nearest-neighbor distance to the behavior-policy samples with its nearest-neighbor distance within the target-policy samples. We then use the ratio between these two distances as a relative measure of support coverage. A ratio close to one suggests that the behavior-policy samples cover the corresponding target region at a similar local scale, whereas a larger ratio indicates weaker local support coverage.

\begin{table}[ht]
\centering
\caption{Nearest-neighbor support coverage diagnostics in the embedding space. The ratio compares the target-to-behavior nearest-neighbor distance with the target-to-target nearest-neighbor distance. A ratio close to one indicates stronger local support coverage.}
\label{tab:support_diagnostic}
\scriptsize
\setlength{\tabcolsep}{5pt}
\begin{tabular}{lcccc}
\toprule
Dataset 
& Target $\to$ Behavior NN 
& Target $\to$ Target NN 
& Ratio 
& Fraction with ratio $>2$ \\
\midrule
\textsc{Chatbot Arena} (GPT-4) 
& 0.855 & 0.871 & 0.98 & 0.0\% \\
\textsc{Chatbot Arena} (GPT-3.5-Turbo) 
& 0.865 & 0.883 & 0.98 & 0.0\% \\
\textsc{RewardBench 2} (Claude 3.5 Sonnet) 
& 0.955 & 0.163 & 5.85 & 94.2\% \\
\textsc{RewardBench 2} (Mixed) 
& 0.936 & 0.164 & 5.71 & 94.1\% \\
\bottomrule
\end{tabular}
\end{table}

As shown in Table~\ref{tab:support_diagnostic}, the ratios on \textsc{Chatbot Arena} are close to one, and no target samples have a ratio greater than two. This provides empirical support for Assumption~\ref{ass:bound_emb} on \textsc{Chatbot Arena}. In contrast, \textsc{RewardBench 2} exhibits weaker overlap, with ratios around $5.7$--$5.9$ and more than $94\%$ of target samples having ratio greater than two. This indicates that the target samples are farther from the behavior-policy samples in the embedding space, which is consistent with the larger estimation errors observed on \textsc{RewardBench 2}. Nevertheless, OTROPE still performs strongly relative to the competing baselines in this more challenging overlap regime.

\subsection{Empirical Verification of Assumption~\ref{ass:ball}}
\label{app:ball}

We empirically examine Assumption~\ref{ass:ball} using real data. Recall that Assumption~\ref{ass:ball} requires the local ball mass $e_{\#}P(B(u,\eta))$ to shrink no faster than $\eta^{d_0}$, where $d_0$ is the intrinsic dimension of the embedding support. 
To estimate $d_0$, for each target-policy sample we compute the fraction of behavior-policy samples that fall within radius $\eta$, average this quantity over all target-policy samples to obtain $\mathrm{mass}(\eta)$, and fit a linear regression between $\log \eta$ and $\log \mathrm{mass}(\eta)$. 
The fitted slope gives an estimate $\widehat d_0$, which measures the effective dimension of the local embedding geometry.

\begin{table}[ht]
\centering
\caption{Estimated intrinsic dimension of the raw pre-PCA bge-m3 embeddings. The slope of the log--log regression gives $\widehat d_0$.}
\label{tab:intrinsic_dim}
\setlength{\tabcolsep}{8pt}
\begin{tabular}{lcc}
\toprule
Dataset & $\widehat d_0$ & $R^2$ \\
\midrule
\textsc{Chatbot Arena} (GPT-4) & 23.98 & 1.0000 \\
\textsc{Chatbot Arena} (GPT-3.5-Turbo) & 24.01 & 1.0000 \\
\textsc{RewardBench 2} (Mixed) & 29.23 & 0.9996 \\
\textsc{RewardBench 2} (Claude 3.5 Sonnet) & 29.17 & 0.9997 \\
\bottomrule
\end{tabular}
\end{table}

Table~\ref{tab:intrinsic_dim} reports the results on the raw pre-PCA bge-m3 embeddings. The nearly perfect log--log linearity, with $R^2 \ge 0.9996$ across all datasets, indicates that the power-law scaling in Assumption~\ref{ass:ball} closely matches the empirical embedding geometry. Moreover, the estimated intrinsic dimensions, $\widehat d_0 \approx 24$--$29$, are far below the raw embedding dimension of the concatenated bge-m3 representations and also below the PCA dimension used in our experiments. These findings provide empirical support for Assumption~\ref{ass:ball} and further justify the use of PCA for dimension reduction.

\subsection{Sensitivity to the PCA Dimension}
\label{app:pca_sensitivity}

Here, we conduct a sensitivity analysis with respect to the embedding dimension. The bge-m3 embedding model produces fixed 1024-dimensional dense embeddings. For pairwise comparisons, we concatenate the embeddings of the candidate and opponent responses, resulting in 2048-dimensional representations. Since bge-m3 does not natively support adjustable embedding dimensions, we vary the effective embedding dimension through PCA truncation. In our main experiments, we use 128-dimensional PCA representations.

To assess the sensitivity of OTROPE to this dimension-reduction choice, we repeat the \textsc{Chatbot Arena} experiments using bge-m3 embeddings while reducing the PCA dimension from 128 to 64. All other hyperparameters and experimental configurations are kept unchanged. Table~\ref{tab:pca64_ablation} reports the results.

\begin{table}[ht]
\centering
\caption{Mean absolute error of the estimated target policy value on \textsc{Chatbot Arena} using 64-dimensional PCA-truncated bge-m3 embeddings. Standard errors are reported in parentheses.}
\label{tab:pca64_ablation}
\setlength{\tabcolsep}{4pt}
\begin{tabular}{lcc|cc}
\toprule
& \multicolumn{2}{c|}{Mixture-of-Experts} 
& \multicolumn{2}{c}{Majority Voting} \\
\cmidrule(lr){2-3} \cmidrule(lr){4-5}
Method & GPT-4 & GPT-3.5 & GPT-4 & GPT-3.5 \\
\midrule
IS     
& 0.408 (0.002) 
& 0.261 (0.002) 
& 0.408 (0.002) 
& 0.261 (0.002) \\

DR     
& 0.088 (0.009) 
& 0.081 (0.007) 
& 0.127 (0.002) 
& 0.110 (0.002) \\

PPI    
& 0.086 (0.009) 
& 0.080 (0.007) 
& 0.124 (0.002) 
& 0.108 (0.002) \\

PPI++  
& 0.302 (0.006) 
& 0.197 (0.004) 
& 0.365 (0.001) 
& 0.234 (0.001) \\

\redcell
OTROPE 
& \textbf{0.049 (0.013)} 
& \textbf{0.076 (0.007)} 
& \textbf{0.088 (0.012)} 
& \textbf{0.106 (0.006)} \\
\bottomrule
\end{tabular}
\end{table}

As expected, reducing the PCA dimension from 128 to 64 increases the MAE of OTROPE relative to the main results in Table~\ref{tab:ot-chat}, since a lower-dimensional representation retains less information and yields a coarser pairwise-cost geometry for OT weighting. Nevertheless, OTROPE still achieves the lowest MAE across all four aggregation--target-policy combinations, indicating that the method remains effective under more aggressive PCA truncation.

\subsection{Sensitivity to the Transport Cost}
\label{app:cost_sensitivity}

For the transport cost, we use the $\ell_2$ distance in all main experiments. To assess sensitivity to this choice, we additionally evaluate OTROPE using the $\ell_1$ distance on \textsc{Chatbot Arena} with GPT-3.5-Turbo as the target policy, while keeping all other hyperparameters and experimental configurations unchanged. Table~\ref{tab:l1_cost_sensitivity} reports the results.

\begin{table}[ht]
\centering
\caption{Mean absolute error on \textsc{Chatbot Arena} under GPT-3.5-Turbo using the $\ell_1$ transport cost. Standard deviations are reported in parentheses.}
\label{tab:l1_cost_sensitivity}
\setlength{\tabcolsep}{6pt}
\begin{tabular}{lcc}
\toprule
Method & Mixture-of-Experts & Majority Voting \\
\midrule
IS     & 0.272 (0.001) & 0.272 (0.001) \\
DR     & 0.069 (0.008) & 0.094 (0.002) \\
PPI    & 0.079 (0.008) & 0.108 (0.002) \\
PPI++  & 0.197 (0.004) & 0.234 (0.001) \\
\redcell OTROPE & \textbf{0.025 (0.011)} & \textbf{0.034 (0.016)} \\
\bottomrule
\end{tabular}
\end{table}

OTROPE achieves the lowest MAE under both aggregation schemes, indicating that its empirical performance is robust to the choice of transport cost. Its performance under the $\ell_1$ cost is also close to that obtained with the default $\ell_2$ cost, suggesting that the proposed correction is not sensitive to this specific metric choice.

\clearpage

\newpage
\section*{NeurIPS Paper Checklist}

\begin{enumerate}

\item {\bf Claims}
    \item[] Question: Do the main claims made in the abstract and introduction accurately reflect the paper's contributions and scope?
    \item[] Answer: \answerYes{} 
    \item[] Justification: We describe our method in the abstract and conclude the contributions in the introduction.
    \item[] Guidelines:
    \begin{itemize}
        \item The answer \answerNA{} means that the abstract and introduction do not include the claims made in the paper.
        \item The abstract and/or introduction should clearly state the claims made, including the contributions made in the paper and important assumptions and limitations. A \answerNo{} or \answerNA{} answer to this question will not be perceived well by the reviewers. 
        \item The claims made should match theoretical and experimental results, and reflect how much the results can be expected to generalize to other settings. 
        \item It is fine to include aspirational goals as motivation as long as it is clear that these goals are not attained by the paper. 
    \end{itemize}

\item {\bf Limitations}
    \item[] Question: Does the paper discuss the limitations of the work performed by the authors?
    \item[] Answer: \answerYes{} 
    \item[] Justification: We discuss the limitations at the end of this paper.
    \item[] Guidelines:
    \begin{itemize}
        \item The answer \answerNA{} means that the paper has no limitation while the answer \answerNo{} means that the paper has limitations, but those are not discussed in the paper. 
        \item The authors are encouraged to create a separate ``Limitations'' section in their paper.
        \item The paper should point out any strong assumptions and how robust the results are to violations of these assumptions (e.g., independence assumptions, noiseless settings, model well-specification, asymptotic approximations only holding locally). The authors should reflect on how these assumptions might be violated in practice and what the implications would be.
        \item The authors should reflect on the scope of the claims made, e.g., if the approach was only tested on a few datasets or with a few runs. In general, empirical results often depend on implicit assumptions, which should be articulated.
        \item The authors should reflect on the factors that influence the performance of the approach. For example, a facial recognition algorithm may perform poorly when image resolution is low or images are taken in low lighting. Or a speech-to-text system might not be used reliably to provide closed captions for online lectures because it fails to handle technical jargon.
        \item The authors should discuss the computational efficiency of the proposed algorithms and how they scale with dataset size.
        \item If applicable, the authors should discuss possible limitations of their approach to address problems of privacy and fairness.
        \item While the authors might fear that complete honesty about limitations might be used by reviewers as grounds for rejection, a worse outcome might be that reviewers discover limitations that aren't acknowledged in the paper. The authors should use their best judgment and recognize that individual actions in favor of transparency play an important role in developing norms that preserve the integrity of the community. Reviewers will be specifically instructed to not penalize honesty concerning limitations.
    \end{itemize}

\item {\bf Theory assumptions and proofs}
    \item[] Question: For each theoretical result, does the paper provide the full set of assumptions and a complete (and correct) proof?
    \item[] Answer: \answerYes{} 
    \item[] Justification: We state the assumptions for the theoretical results and provide complete proofs in Appendix~\ref{app:theoretical_analysis}.
    \item[] Guidelines:
    \begin{itemize}
        \item The answer \answerNA{} means that the paper does not include theoretical results. 
        \item All the theorems, formulas, and proofs in the paper should be numbered and cross-referenced.
        \item All assumptions should be clearly stated or referenced in the statement of any theorems.
        \item The proofs can either appear in the main paper or the supplemental material, but if they appear in the supplemental material, the authors are encouraged to provide a short proof sketch to provide intuition. 
        \item Inversely, any informal proof provided in the core of the paper should be complemented by formal proofs provided in appendix or supplemental material.
        \item Theorems and Lemmas that the proof relies upon should be properly referenced. 
    \end{itemize}

    \item {\bf Experimental result reproducibility}
    \item[] Question: Does the paper fully disclose all the information needed to reproduce the main experimental results of the paper to the extent that it affects the main claims and/or conclusions of the paper (regardless of whether the code and data are provided or not)?
    \item[] Answer: \answerYes{} 
    \item[] Justification: In experiments section and Appendix~\ref{app:exp}, we disclose the information of our experiments including both controlled experiments and real data applications on LLM evaluations.
    \item[] Guidelines:
    \begin{itemize}
        \item The answer \answerNA{} means that the paper does not include experiments.
        \item If the paper includes experiments, a \answerNo{} answer to this question will not be perceived well by the reviewers: Making the paper reproducible is important, regardless of whether the code and data are provided or not.
        \item If the contribution is a dataset and\slash or model, the authors should describe the steps taken to make their results reproducible or verifiable. 
        \item Depending on the contribution, reproducibility can be accomplished in various ways. For example, if the contribution is a novel architecture, describing the architecture fully might suffice, or if the contribution is a specific model and empirical evaluation, it may be necessary to either make it possible for others to replicate the model with the same dataset, or provide access to the model. In general. releasing code and data is often one good way to accomplish this, but reproducibility can also be provided via detailed instructions for how to replicate the results, access to a hosted model (e.g., in the case of a large language model), releasing of a model checkpoint, or other means that are appropriate to the research performed.
        \item While NeurIPS does not require releasing code, the conference does require all submissions to provide some reasonable avenue for reproducibility, which may depend on the nature of the contribution. For example
        \begin{enumerate}
            \item If the contribution is primarily a new algorithm, the paper should make it clear how to reproduce that algorithm.
            \item If the contribution is primarily a new model architecture, the paper should describe the architecture clearly and fully.
            \item If the contribution is a new model (e.g., a large language model), then there should either be a way to access this model for reproducing the results or a way to reproduce the model (e.g., with an open-source dataset or instructions for how to construct the dataset).
            \item We recognize that reproducibility may be tricky in some cases, in which case authors are welcome to describe the particular way they provide for reproducibility. In the case of closed-source models, it may be that access to the model is limited in some way (e.g., to registered users), but it should be possible for other researchers to have some path to reproducing or verifying the results.
        \end{enumerate}
    \end{itemize}

\item {\bf Open access to data and code}
    \item[] Question: Does the paper provide open access to the data and code, with sufficient instructions to faithfully reproduce the main experimental results, as described in supplemental material?
    \item[] Answer: \answerYes{} 
    \item[] Justification: We will release the code to reproduce our main experimental results under open source upon acceptance. During the review process the code is shared in supplementary material.
    \item[] Guidelines:
    \begin{itemize}
        \item The answer \answerNA{} means that paper does not include experiments requiring code.
        \item Please see the NeurIPS code and data submission guidelines (\url{https://neurips.cc/public/guides/CodeSubmissionPolicy}) for more details.
        \item While we encourage the release of code and data, we understand that this might not be possible, so \answerNo{} is an acceptable answer. Papers cannot be rejected simply for not including code, unless this is central to the contribution (e.g., for a new open-source benchmark).
        \item The instructions should contain the exact command and environment needed to run to reproduce the results. See the NeurIPS code and data submission guidelines (\url{https://neurips.cc/public/guides/CodeSubmissionPolicy}) for more details.
        \item The authors should provide instructions on data access and preparation, including how to access the raw data, preprocessed data, intermediate data, and generated data, etc.
        \item The authors should provide scripts to reproduce all experimental results for the new proposed method and baselines. If only a subset of experiments are reproducible, they should state which ones are omitted from the script and why.
        \item At submission time, to preserve anonymity, the authors should release anonymized versions (if applicable).
        \item Providing as much information as possible in supplemental material (appended to the paper) is recommended, but including URLs to data and code is permitted.
    \end{itemize}

\item {\bf Experimental setting/details}
    \item[] Question: Does the paper specify all the training and test details (e.g., data splits, hyperparameters, how they were chosen, type of optimizer) necessary to understand the results?
    \item[] Answer: \answerYes{} 
    \item[] Justification: All the experimental details are provided in Appendix~\ref{app:exp}.
    \item[] Guidelines:
    \begin{itemize}
        \item The answer \answerNA{} means that the paper does not include experiments.
        \item The experimental setting should be presented in the core of the paper to a level of detail that is necessary to appreciate the results and make sense of them.
        \item The full details can be provided either with the code, in appendix, or as supplemental material.
    \end{itemize}

\item {\bf Experiment statistical significance}
    \item[] Question: Does the paper report error bars suitably and correctly defined or other appropriate information about the statistical significance of the experiments?
    \item[] Answer: \answerYes{} 
    \item[] Justification: We provide mean absolute error with the standard deviation under multiple replications.
    \item[] Guidelines:
    \begin{itemize}
        \item The answer \answerNA{} means that the paper does not include experiments.
        \item The authors should answer \answerYes{} if the results are accompanied by error bars, confidence intervals, or statistical significance tests, at least for the experiments that support the main claims of the paper.
        \item The factors of variability that the error bars are capturing should be clearly stated (for example, train/test split, initialization, random drawing of some parameter, or overall run with given experimental conditions).
        \item The method for calculating the error bars should be explained (closed form formula, call to a library function, bootstrap, etc.)
        \item The assumptions made should be given (e.g., Normally distributed errors).
        \item It should be clear whether the error bar is the standard deviation or the standard error of the mean.
        \item It is OK to report 1-sigma error bars, but one should state it. The authors should preferably report a 2-sigma error bar than state that they have a 96\% CI, if the hypothesis of Normality of errors is not verified.
        \item For asymmetric distributions, the authors should be careful not to show in tables or figures symmetric error bars that would yield results that are out of range (e.g., negative error rates).
        \item If error bars are reported in tables or plots, the authors should explain in the text how they were calculated and reference the corresponding figures or tables in the text.
    \end{itemize}

\item {\bf Experiments compute resources}
    \item[] Question: For each experiment, does the paper provide sufficient information on the computer resources (type of compute workers, memory, time of execution) needed to reproduce the experiments?
    \item[] Answer: \answerYes{} 
    \item[] Justification: Yes. We report the computational resources used for the experiments, including the type of GPU and its memory capacity in Appendix~\ref{app:exp}.
    \item[] Guidelines:
    \begin{itemize}
        \item The answer \answerNA{} means that the paper does not include experiments.
        \item The paper should indicate the type of compute workers CPU or GPU, internal cluster, or cloud provider, including relevant memory and storage.
        \item The paper should provide the amount of compute required for each of the individual experimental runs as well as estimate the total compute. 
        \item The paper should disclose whether the full research project required more compute than the experiments reported in the paper (e.g., preliminary or failed experiments that didn't make it into the paper). 
    \end{itemize}
    
\item {\bf Code of ethics}
    \item[] Question: Does the research conducted in the paper conform, in every respect, with the NeurIPS Code of Ethics \url{https://neurips.cc/public/EthicsGuidelines}?
    \item[] Answer: \answerYes{} 
    \item[] Justification: The research complies with the NeurIPS Code of Ethics. It uses only publicly available datasets with proper citations, does not involve sensitive personal data, and releases code responsibly. 
    \item[] Guidelines: 
    \begin{itemize}
        \item The answer \answerNA{} means that the authors have not reviewed the NeurIPS Code of Ethics.
        \item If the authors answer \answerNo, they should explain the special circumstances that require a deviation from the Code of Ethics.
        \item The authors should make sure to preserve anonymity (e.g., if there is a special consideration due to laws or regulations in their jurisdiction).
    \end{itemize}

\item {\bf Broader impacts}
    \item[] Question: Does the paper discuss both potential positive societal impacts and negative societal impacts of the work performed?
    \item[] Answer: \answerNA{} 
    \item[] Justification: The work is primarily methodological, focusing on the evaluation of LLMs. We do not identify direct positive or negative societal impacts.
    \item[] Guidelines:
    \begin{itemize}
        \item The answer \answerNA{} means that there is no societal impact of the work performed.
        \item If the authors answer \answerNA{} or \answerNo, they should explain why their work has no societal impact or why the paper does not address societal impact.
        \item Examples of negative societal impacts include potential malicious or unintended uses (e.g., disinformation, generating fake profiles, surveillance), fairness considerations (e.g., deployment of technologies that could make decisions that unfairly impact specific groups), privacy considerations, and security considerations.
        \item The conference expects that many papers will be foundational research and not tied to particular applications, let alone deployments. However, if there is a direct path to any negative applications, the authors should point it out. For example, it is legitimate to point out that an improvement in the quality of generative models could be used to generate Deepfakes for disinformation. On the other hand, it is not needed to point out that a generic algorithm for optimizing neural networks could enable people to train models that generate Deepfakes faster.
        \item The authors should consider possible harms that could arise when the technology is being used as intended and functioning correctly, harms that could arise when the technology is being used as intended but gives incorrect results, and harms following from (intentional or unintentional) misuse of the technology.
        \item If there are negative societal impacts, the authors could also discuss possible mitigation strategies (e.g., gated release of models, providing defenses in addition to attacks, mechanisms for monitoring misuse, mechanisms to monitor how a system learns from feedback over time, improving the efficiency and accessibility of ML).
    \end{itemize}
    
\item {\bf Safeguards}
    \item[] Question: Does the paper describe safeguards that have been put in place for responsible release of data or models that have a high risk for misuse (e.g., pre-trained language models, image generators, or scraped datasets)?
    \item[] Answer: \answerNA{} 
    \item[] Justification: This work does not release new datasets or models with high risk for misuse. It uses only publicly available datasets and focuses on LLM evaluation. Therefore, safeguards for high-risk releases are not applicable.
    \item[] Guidelines:
    \begin{itemize}
        \item The answer \answerNA{} means that the paper poses no such risks.
        \item Released models that have a high risk for misuse or dual-use should be released with necessary safeguards to allow for controlled use of the model, for example by requiring that users adhere to usage guidelines or restrictions to access the model or implementing safety filters. 
        \item Datasets that have been scraped from the Internet could pose safety risks. The authors should describe how they avoided releasing unsafe images.
        \item We recognize that providing effective safeguards is challenging, and many papers do not require this, but we encourage authors to take this into account and make a best faith effort.
    \end{itemize}

\item {\bf Licenses for existing assets}
    \item[] Question: Are the creators or original owners of assets (e.g., code, data, models), used in the paper, properly credited and are the license and terms of use explicitly mentioned and properly respected?
    \item[] Answer: \answerYes{} 
    \item[] Justification: All datasets and models used in this work are publicly available and properly credited with appropriate citations. We adhere to the respective licenses and terms of use of these resources and ensure that they are used in accordance with their intended purposes.
    \item[] Guidelines:
    \begin{itemize}
        \item The answer \answerNA{} means that the paper does not use existing assets.
        \item The authors should cite the original paper that produced the code package or dataset.
        \item The authors should state which version of the asset is used and, if possible, include a URL.
        \item The name of the license (e.g., CC-BY 4.0) should be included for each asset.
        \item For scraped data from a particular source (e.g., website), the copyright and terms of service of that source should be provided.
        \item If assets are released, the license, copyright information, and terms of use in the package should be provided. For popular datasets, \url{paperswithcode.com/datasets} has curated licenses for some datasets. Their licensing guide can help determine the license of a dataset.
        \item For existing datasets that are re-packaged, both the original license and the license of the derived asset (if it has changed) should be provided.
        \item If this information is not available online, the authors are encouraged to reach out to the asset's creators.
    \end{itemize}

\item {\bf New assets}
    \item[] Question: Are new assets introduced in the paper well documented and is the documentation provided alongside the assets?
    \item[] Answer: \answerNA{} 
    \item[] Justification: The work does not release any new datasets, and the proposed policy evaluation framework is not provided as a standalone asset.
    \item[] Guidelines:
    \begin{itemize}
        \item The answer \answerNA{} means that the paper does not release new assets.
        \item Researchers should communicate the details of the dataset\slash code\slash model as part of their submissions via structured templates. This includes details about training, license, limitations, etc. 
        \item The paper should discuss whether and how consent was obtained from people whose asset is used.
        \item At submission time, remember to anonymize your assets (if applicable). You can either create an anonymized URL or include an anonymized zip file.
    \end{itemize}

\item {\bf Crowdsourcing and research with human subjects}
    \item[] Question: For crowdsourcing experiments and research with human subjects, does the paper include the full text of instructions given to participants and screenshots, if applicable, as well as details about compensation (if any)? 
    \item[] Answer: \answerNA{} 
    \item[] Justification: This work does not involve any new crowdsourcing experiments or direct interaction with human subjects. All human preference data are obtained from existing publicly available datasets, and no new data collection, participant instructions, or compensation are involved.
    \item[] Guidelines:
    \begin{itemize}
        \item The answer \answerNA{} means that the paper does not involve crowdsourcing nor research with human subjects.
        \item Including this information in the supplemental material is fine, but if the main contribution of the paper involves human subjects, then as much detail as possible should be included in the main paper. 
        \item According to the NeurIPS Code of Ethics, workers involved in data collection, curation, or other labor should be paid at least the minimum wage in the country of the data collector. 
    \end{itemize}

\item {\bf Institutional review board (IRB) approvals or equivalent for research with human subjects}
    \item[] Question: Does the paper describe potential risks incurred by study participants, whether such risks were disclosed to the subjects, and whether Institutional Review Board (IRB) approvals (or an equivalent approval/review based on the requirements of your country or institution) were obtained?
    \item[] Answer: \answerNA{} 
    \item[] Justification: This work does not involve new studies with human participants. It relies solely on publicly available datasets with pre-collected human annotations, and no direct interaction with human subjects is conducted. Therefore, considerations such as participant risk, disclosure, or IRB approval are not applicable.
    \item[] Guidelines:
    \begin{itemize}
        \item The answer \answerNA{} means that the paper does not involve crowdsourcing nor research with human subjects.
        \item Depending on the country in which research is conducted, IRB approval (or equivalent) may be required for any human subjects research. If you obtained IRB approval, you should clearly state this in the paper. 
        \item We recognize that the procedures for this may vary significantly between institutions and locations, and we expect authors to adhere to the NeurIPS Code of Ethics and the guidelines for their institution. 
        \item For initial submissions, do not include any information that would break anonymity (if applicable), such as the institution conducting the review.
    \end{itemize}

\item {\bf Declaration of LLM usage}
    \item[] Question: Does the paper describe the usage of LLMs if it is an important, original, or non-standard component of the core methods in this research? Note that if the LLM is used only for writing, editing, or formatting purposes and does \emph{not} impact the core methodology, scientific rigor, or originality of the research, declaration is not required.
    \item[] Answer: \answerNA{} 
    \item[] Justification: LLMs were used only for writing, editing, or formatting purposes and do not affect the core methodology, scientific rigor, or originality of the research.
    \item[] Guidelines:
    \begin{itemize}
        \item The answer \answerNA{} means that the core method development in this research does not involve LLMs as any important, original, or non-standard components.
        \item Please refer to our LLM policy in the NeurIPS handbook for what should or should not be described.
    \end{itemize}

\end{enumerate}

\end{document}